\documentclass[twoside,11pt]{article}
\usepackage[preprint]{jmlr2e}
\usepackage[utf8]{inputenc}
\usepackage{caption}
\usepackage{subcaption}

\newcommand{\altL}{L}
\newcommand{\altI}{I}

\usepackage[most]{tcolorbox}
\usepackage{times}
\usepackage{tikz-cd}
\usepackage{tikz}
\usepackage{amsmath}
\usepackage{bm}
\usepackage[T1]{fontenc}    
\usepackage{url}            
\usepackage{booktabs}       

\usepackage{nicefrac}       
\usepackage{microtype}      
\usepackage{algorithm,algorithmic}

\newcommand{\bn}{\bm{\|}}
\newcommand{\bl}{\bm{\langle}}
\newcommand{\br}{\bm{\rangle}_{2}}

\newcommand{\mb}[1]{\mathbf{#1}}
\newcommand{\E}{{E}}
\newcommand{\cH}{{\mathcal{H}}}
\newcommand{\B}{{\mathcal{B}}}
\newcommand{\mS}{\bm{S}}
\newcommand{\mX}{\phi(X)}
\newcommand{\mY}{\bm{Y}}
\newcommand{\mZ}{\bm{Z}}

\newcommand{\cO}{\mathcal{O}}
\newcommand{\cF}{\mathcal{F}}
\newcommand{\cM}{\mathcal{M}}
\newtheorem{prop}{Proposition}
\newtheorem{lem}{Lemma}

\newtheorem{prob}{Problem}
\newtheorem{defn}{Definition}

\DeclareMathOperator*{\cch}{cch}

\DeclareMathOperator*{\argmin}{arg\,min}

\jmlrheading{}{}{}{}{}{}{}

\ShortHeadings{Oblivious Data for Fairness with Kernels}{Gr\"unew\"alder and Khaleghi}
\firstpageno{1}

\begin{document}

\title{Oblivious Data for Fairness with Kernels}

\author{\name Steffen Gr\"unew\"alder \email s.grunewalder@lancaster.ac.uk 
\AND
\name Azadeh Khaleghi \email a.khaleghi@lancaster.ac.uk 
\\
       \addr Department of Mathematics and Statistics \\
       Lancaster University\\
       Lancaster,  UK}

\editor{}

\maketitle

\begin{abstract}
We investigate the problem of algorithmic fairness in the case where sensitive and non-sensitive features are available and one aims to generate new, `oblivious', features that closely approximate the non-sensitive features, and are only minimally dependent on the sensitive ones. We study this question in the context of kernel methods. We analyze a relaxed version of the Maximum Mean Discrepancy criterion which does not guarantee full independence but makes the optimization problem tractable. We derive a closed-form solution for this relaxed optimization problem and complement the result with a study of the dependencies between the newly generated features and the sensitive ones. Our key ingredient for generating such oblivious features is a Hilbert-space-valued conditional expectation, which needs to be estimated from data. We propose a plug-in approach and demonstrate how the estimation errors can be controlled. While our techniques help reduce the bias, we would like to point out that  no post-processing of any dataset could possibly serve as an alternative to well-designed experiments. 
\end{abstract}


\section{Introduction}
Machine learning algorithms trained on historical data may inherit implicit biases which can in turn lead to potentially unfair outcomes for some individuals or minority groups. For instance, gender-bias may be present in a historical dataset on which a model is trained to automate the postgraduate admission process at a university. This may in turn render the algorithm biased, leading it to 
inadvertently generate unfair decisions. 
In recent years, a large body of work has been dedicated to systematically addressing this problem, whereby various notions of fairness have been considered, see, e.g. \citep{CAL09, ZEM13, LOU15, HAR16,JOS16, KIL17,KUS17,CAL17,ZAF17,KLE17,DON18, MAD18}, and references therein. 

Among the several {\em algorithmic fairness} criteria, one important objective is to ensure that a model's prediction is not influenced by the presence of sensitive information in the data.  
In this paper, we address this objective from the perspective of (fair) representation learning. Thus, a central question which forms the basis of our work is as follows.
\begin{center}
{\em Can the observed features be replaced by close approximations \\ that are independent of the sensitive ones?}
\end{center}
More formally, 
assume that we have a dataset such that each data-point is a realization of a random variable $(X,S)$ where $S$ and $X$ are in turn vector-valued random variables corresponding to the sensitive and non-sensitive features respectively. We further allow $X$ and $S$ to be arbitrarily dependent, and ask whether it is possible to generate a new random variable $Z$ which is ideally independent of $S$ and close to $X$ in some {\em meaningful probabilistic sense}. 
This objective is 
As an initial step, we may assume that $X$ is zero-mean, and aim for decorrelation between $Z$ and $X$.  
This can be achieved by letting 
$Z = X - E^S X$ 
where $E^S X$ is the conditional expectation of $X$ given $S$.
The random variable $Z$ so-defined is not correlated with $S$ 
and is close to $X$. In particular, it recovers $X$ if $X$ and $S$ are independent. 
In fact, under mild assumptions, $Z$ gives the best approximation (in the mean-squared sense) of $X$, while being uncorrelated with $S$. 
Observe that while the distribution of $Z$ differs from that of $X$, this new random variable seems to serve the purpose well. For instance, if $S$ corresponds to a subject's {\em gender} and $X$ to a subject's {\em height}, then $Z$ corresponds to height of the subject centered around the average height of the class corresponding to the subject's gender. The key contributions of this work, briefly summarized below, are theoretical; we also provide an evaluation of the proposed approach through experiments in the context of classification and regression\footnote{Our implementations are available at \url{https://github.com/azalk/Oblivious.git}.}. Before giving an overview of our results, we would also like to point out that while our techniques help reduce the bias, it is important to note that no post-processing of any dataset could possibly serve as an alternative to well-designed experiments. 

\paragraph{Contributions.} 
Building upon this intuition, and using results inspired by testing for independence using the Maximum Mean Discrepancy (MMD) criterion (see e.g.~\citet{GRET08}), 
we obtain a related optimization problem in which $X$ and $E^S X$ are replaced with Hilbert-space-valued random variables and Hilbert-space-valued conditional expectations. 
While the move to Hilbert spaces does not enforce complete independence between the new features and the sensitive features, it helps to  significantly reduce the dependencies between the features. The new features $\mZ$ have various useful properties which we explore in this paper. They are also easy to generate from samples $(X_1,S_1),\ldots, (X_n,S_n)$. The main challenge in generating the oblivious features $\mZ_1,\ldots,\mZ_n$ is that we do not have access to the  Hilbert-space-valued conditional expectation and need to estimate it from data. Since we are concerned with Reproducing Kernel Hilbert Spaces (RKHSs) here, we use the reproducing property to extend the plugin approach of \citet{SG} to the RKHS setting and tackle the estimation problem. We further show how estimation errors can be controlled. Having obtained the empirical estimates of the conditional expectations, we generate oblivious features and an oblivious kernel matrix to be used as input to any kernel method. 
This guarantees a significant reduction in the dependence between the predictions and the sensitive features.
We cast the objective of finding oblivious features $\mZ$ which approximate the original features $X$ well while maintaining minimal  dependence on the sensitive features $S$, as a constrained optimization problem. 
Making use of Hilbert-space-valued conditional expectations, we provide a closed form solution to the optimization problem proposed. Specifically, we first prove in that our solution satisfies the constraint of the optimization problem at hand, and show via Proposition \ref{lem:best_approx} that it is indeed optimal.
Through Proposition \ref{prop:alpha_dep} we relate the strength of the dependencies between $\mZ$ and $S$ to how close $\mZ$ lies to the low-dimensional manifold corresponding to the image under the feature map $\phi$.
This result is key in providing some insight into the interplay between probabilistic independence and approximations in the Hilbert space.
We extend known estimators for real-valued conditional expectations to estimate those taking values in a Hilbert space, and show via Proposition \ref{Prop:Est_error_cond} how to control their estimation errors. This result in itself may be of independent interest in future research concerning  Hilbert-space-valued conditional expectations. 
We provide a method to generate  oblivious features and the oblivious kernel matrix which can be used instead of the kernel matrix to reduce the dependence of the prediction on the sensitive features; the computational complexity of the approach is $O(n^2)$.

\paragraph{Related Work.} 
Among the vast literature on algorithmic fairness, \citet{DON18, MAD18},  which fit into the larger body of work on fair representation learning, are closest to our approach. 
\citet{MAD18} describe a general framework for fair representation learning. The approach taken is inspired by generative adversarial networks and is based on a game played between generative models and adversarial evaluations. Depending on which function classes one considers for the generative models and for the adversarial evaluations one can describe a vast array of approaches. Interestingly, it is possible to interpret our approach in this general context: the encoder $f$ corresponds to a map from $\mathbb{X}$ and $\mathbb{S}$ to $\cH$, where our new features $\mZ$ live. We do not have a decoder but compare features directly (one could also take our decoder to be the identity map). Our adversary is different from that used by \citet{MAD18}. In their approach a regressor is inferred which maps the features to the sensitive features, while we compare sensitive features and new features by applying test functions to them. The regression approach performs well in their context because they only consider finitely many sensitive features. In the more general framework considered in the present paper where the sensitive features are allowed to take on continuous values, this approach would be sub-optimal since it cannot capture all dependencies.
Finally, we ignore labels when inferring new features. 
It is also worth pointing out that our approach is not based on a game played between generative models and an adversary but we provide closed form solutions.
On other hand, while the focus of \citet{DON18} is mostly on empirical risk minimization under fairness constraints, the authors briefly discuss representation learning for fairness as well. In particular, Equation (13) in the reference paper effectively describes  a conditional expectation in Hilbert space, though it is not denoted or motivated as such. The conditional expectation is based on the binary features $S$ only and the construction is applied in the linear kernel context to derive new features. The authors do not go beyond the linear case for representation learning but there is a clear link to the more general notions of conditional expectation on which we base our work. We discuss the relation to \citet{DON18} in detail in Section \ref{sec:disc_don} and we show how their approach can be extended beyond binary sensitive features by making use of our conditional expectation estimates. 

\paragraph{Organization.}
The rest of the paper is organized as follows. In Section~\ref{sec:pre} we introduce our notation 
and provide preliminary definitions used in the paper. 
Our problem formulation and optimization objective are stated in Section \ref{sec:prob}. 
As part of the formulation we also define the notion of $\cH$-independence between Hilbert-space-valued features  and the sensitive features. In Section \ref{sec:bounds_on_dep} we study the relation between $\cH$-independence and bounds on the dependencies between oblivious and sensitive features. In Section \ref{sec:best_ind_feat} we provide a solution to the optimization objective. In Section \ref{sec:est_cond_exp} we derive an estimator for the conditional expectation and use it to generate oblivious features and the oblivious kernel matrix. We provide some empirical evaluations in Section~\ref{sec:experiments}. 
\section{Preliminaries}\label{sec:pre}
In this section we introduce some notation and basic definitions. 
Consider a probability space $(\Omega,\mathcal{A},P)$. For any $A \in \mathcal A$ we let $\chi {A}:\Omega \rightarrow \{0,1\}$ be the indicator function such that $\chi A(\omega) = 1$ if, and only if, $\omega \in A$. 
Let $\mathbb X$ be a measurable space in which a random variable $X: \Omega \rightarrow \mathbb X$ takes values. 
We denote by $\sigma(X)$ the $\sigma$-algebra generated by $X$. 
Let $\cH$ be an RKHS composed of functions $h:\mathbb{X} \rightarrow \mathbb{R}$ and denote its feature map by $\phi(x): \mathbb X \rightarrow \cH$ where, $\phi(x)=k(x,\cdot)$ for some positive definite kernel $k: \mathbb X \times \mathbb X \rightarrow \mathbb R$. As follows from the reproducing kernel property of $\cH$ we have $\langle \phi(x), h \rangle = h(x)$  for all $h \in \cH$.  
Moreover, observe that $\phi(X)$ is in turn a random variable attaining values in $\cH$. In Appendix \ref{app:prob_in_H} we provide some technical details concerning Hilbert-space-valued random variables such as $\phi(X)$.
\paragraph{Conditional Expectation.} 
Let $S: \Omega \rightarrow \mathbb S$ be a random variable taking values in a measurable space $\mathbb S$. 
For the random variable $X$ defined above, we denote by $E^S X$ the random variable corresponding to Kolmogorov's conditional expectation of $X$ given $S$, i.e. $E^S X = E (X|\sigma(S))$, see, e.g. \cite{SHI89}. Recall that in a special case where $\mathbb S = \{0,1\}$ we simply have
\[
E(X |S=0) \chi \{S=0\} + E(X |S=1) \chi \{S=1\}
\]
where, $E(X|S=i)$ is the familiar conditional expectation of  $X$ given the event $\{S=i\}$ for $i=0,1$. Thus, in this case, the random variable $E^S X$ is equal to $E(X|S=0)$ if $S$ attains value $0$ and is equal to $E(X|S=1)$ otherwise. Note that the above example is for illustration only, and that $X$ and $S$ may be arbitrary random variables: they are not required to be binary or discrete-valued. Unless otherwise stated, in this paper we use Kolmogorov's notion of conditional expectation.
We will also be concerned with conditional expectations that attain values in a Hilbert space $\cH$, which mostly behave like real-valued conditional expectations (see \citet{PIS16} and Appendix~\ref{sec:HS_space_valued} for details).
Next, we introduce Hilbert-space-valued $\mathcal{\altL}^2$-spaces which play a prominent role in our results.
\paragraph{Hilbert-space-valued $\mathcal{\altL}^2$-spaces.} 
For a Hilbert space $\cH$, we denote by $\mathcal{\altL}^2(\cH)= \mathcal{\altL}^2(\Omega,\mathcal{A},P;\cH)$ the $\cH$-valued $\mathcal{\altL}^2$ space. If $\cH$ is an RKHS with a bounded and measurable kernel function then $\phi(X)$ is an element of $\mathcal{\altL}^2(\Omega,\mathcal{A},P;\cH)$. 
The space $\mathcal{\altL}^2(\Omega,\mathcal{A},P;\cH)$ consists of all (Bochner)-measurable functions $\mX$ from $\Omega$ to $\cH$ such that $E(\|\mX\|^2)<\infty$  (see Appendix~\ref{app:prob_in_H} for more details). We call these functions  random variables or Hilbert-space-valued random variables and denote them with bold capital letters. 
As in the scalar case we have a corresponding space of equivalence classes which we denote by $L^2(\Omega,\mathcal{A},P;\cH)$. For $\mX,\mY \in
\mathcal{\altL}^2(\Omega,\mathcal{A},P;\cH)$ we use $\mX^\bullet,\mY^\bullet$ for the corresponding equivalence classes in $L^2(\Omega,\mathcal{A},P;\cH)$. The space $L^2(\Omega,\mathcal{A},P;\cH)$  is itself a Hilbert space with norm and inner product given by
$\bn \mX^\bullet \bn_{2}^2 = E(\|\mX \|^2)$ and  $\bl \mX^\bullet, \mY^\bullet \br = E(\langle \mX,  \mY \rangle)$, 
where we use a subscript to distinguish this  norm and inner product from the ones from $\cH$. The norm and inner product have a corresponding pseudo-norm and bilinear form acting on $\mathcal{\altL}^2(\cH)$ and we also denote these by $\bn \cdot \bn_{2}$ and     $\bl \cdot, \cdot \br$.

\section{Problem Formulation}\label{sec:prob}
We formulate the problem as follows. 
Given two random variables $X: \Omega \rightarrow \mathbb X$ and $S: \Omega \rightarrow \mathbb S$ corresponding to non-sensitive and sensitive features in a dataset, we wish to devise a random variable $Z: \Omega \rightarrow \mathbb X$ which is independent of $S$ and closely approximates $X$ in the sense that for all $Z': \Omega \rightarrow \mathbb X$ we have,
\begin{equation}\label{eq:obj1}
    \| Z-X \|_2 \leq \| Z'-X \|_2. 
\end{equation}
Dependencies between random variables can be very subtle and difficult to detect. Similarly, completely removing the dependence of $X$ on $S$ without changing $X$ drastically is an intricate task that is rife with difficulties. Thus, we aim for a more tractable objective, described below, which still gives us control over the dependencies. 

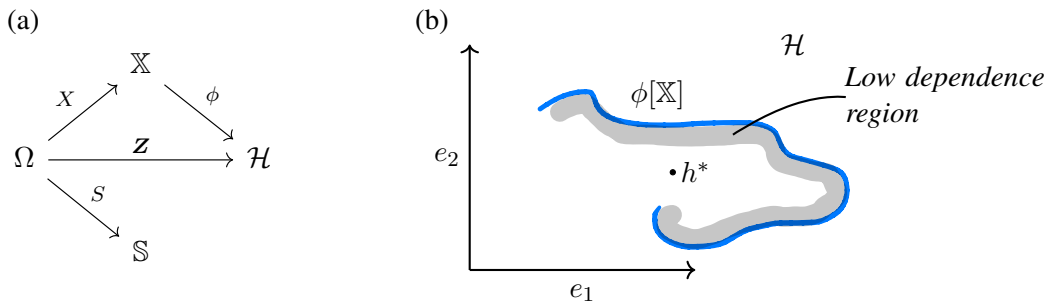
\begin{figure}[t]
\begin{flushleft}    
\begin{tikzpicture}
\node at (0.3,1.8) {(a)};
    \begin{tikzcd}
                             & \mathbb{X} \arrow[dr,"\phi"] &  \\
     \Omega \arrow[ur,"X"] \arrow[rr,"\mZ"] \arrow[dr,"S"] &  & \cH \\
     &  \mathbb{S} &   
    \end{tikzcd}
      \begin{tikzcd}[cramped,sep=small]
    \quad & \quad
    \end{tikzcd}
    \node at (0.5,1.8) {(b)};
    \node[anchor=south west,inner sep=0] at (2,1) {\includegraphics[trim={5cm, 18cm, 3cm, 5cm}, width=0.6 \textwidth]{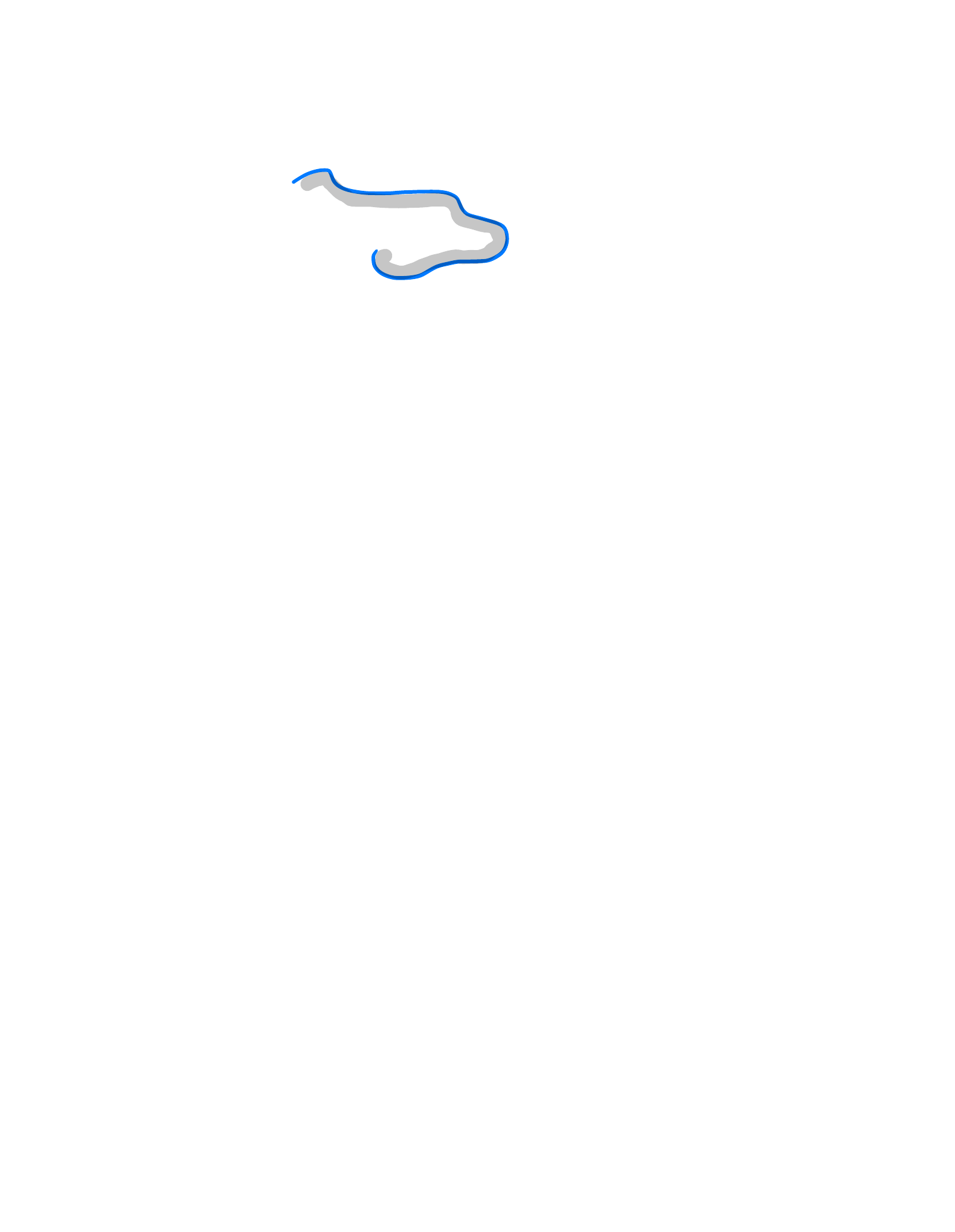}};
    \draw[->,thick] (1,-1.5) -- (1,1.5);
\draw (0.7,0) node {$e_2$};
\draw (2.5,-1.8) node {$e_1$};
\draw[->,thick] (1,-1.5) -- (4,-1.5);
\draw (5.3,1.5) node  {$\cH$}; 
\draw [thick] plot [smooth, tension=1] coordinates { (4.5,0.3) (5.3,0.7) (6, 0.8)};
\draw (7.5,0.8) node[text width=3cm] {\textit{Low dependence region}};
\draw (3.5,0.5) node[above] {$\phi[\mathbb{X}]$};
\draw[fill] (3.7,-0.2) circle (0.2ex); 
\draw (4,-0.2) node {$h^*$}; 
\end{tikzpicture}
\end{flushleft}
\caption{\textbf{(a)} The three main random variables in Problem~\ref{prob:main} are shown. The non-sensitive features $X$ attains values in $\mathbb{X}$ and is mapped onto the RKHS $\cH$ through the feature map $\phi$; the sensitive features $S$ attains values in $\mathbb{S}$, and $\mZ$ attains values in $\cH$. All three random variables are defined on the same probability space $(\Omega,\mathcal{A},P)$.
\textbf{(b)} The image of $\mathbb{X}$ under $\phi$  is sketched (blue curve). This is a subset of $\cH$ whose projection onto the subspace spanned by two orthonormal basis elements  $e_1$ and $e_2$ is shown here.  The set $\phi[\mathbb{X}]$  is a low-dimensional manifold if $\phi$ is continuous. The element $h^* = E(\phi(X))$ lies in the convex hull of $\phi[\mathbb{X}]$. Intuitively, if $\mZ$ attains values mainly in the gray shaded area then $\mZ$ is only weakly dependent on $S$.
} \label{fig:dependence_bound}
\end{figure}

We start by a  \textit{strategic shift} from probabilistic concepts to interactions between functions and random variables. Consider the RKHS $\cH$ of functions $h:\mathbb{X} \rightarrow \mathbb{R}$ with feature map $\phi$ as introduced in Section~\ref{sec:pre}, and assume that $\cH$ is large enough to allow for the approximation of arbitrary indicator functions $\chi \{Z \in A'\}$ in the $\mathcal{\altL}^2$-pseudo-norm for any $\mathbb{X}$-valued random variable $Z$.
Observe that if
\begin{equation} \label{eq:independence}
E(h(Z) \times g(S)) = E(h(Z))  \cdot E(g(S))
\end{equation}
for all $h\in\cH, g\in \mathcal{\altL}^2$ then $Z$ and $S$ are, indeed, independent. 
This is because $h$ and $g$ can be used to approximate arbitrary indicator functions, which together with \eqref{eq:independence} gives,
\begin{align*}
P( \{Z\in A'\} \cap \{S\in B'\}) \approx E(h(Z) \times g(S)) 
= E(h(Z))  \cdot E(g(S)) \approx P(Z\in A') \cdot P(S\in B').
\end{align*}
This means that the independence constraint of the optimization problem of \eqref{eq:obj1} translates to \eqref{eq:independence}. Note that using RKHS elements as test functions is a common approach for detecting dependencies and is used in the MMD-criterion  (e.g.~\citet{GRET08}).

On the other hand, due to the reproducing property of the kernel of $\cH$, we can also rewrite the constraint
\eqref{eq:independence} as 
\begin{equation}\label{eq:const}
E(\langle h, \phi(Z) \rangle \times g(S)) = E \langle h,\phi(Z) \rangle \cdot E(g(S)).
\end{equation}
Observe that $\phi(Z)$ is a random variable that attains values in a low-dimensional manifold; if the kernel function is continuous and $\mathbb{X} = \mathbb{R}^d$ then the image $\phi[\mathbb{X}]$  of $\mathbb{X}$ under $\phi$ is a $d$-dimensional manifold which we denote in the following by $\mathcal{M}$. In Figure \ref{fig:dependence_bound} this manifold is visualized as the blue curve. 
Therefore, while Equation \eqref{eq:const} is linear in $\phi(Z)$, depending on the shape of the manifold, it can lead to an arbitrarily complex optimization problem. 

We propose to relax \eqref{eq:const} by moving away from the manifold, replacing $\phi(Z)$ with a random variable $\mZ: \Omega \rightarrow \cH$ 
which potentially has all of $\cH$ as its range. This simplifies the original optimization problem to one over a vector space under a linear constraint. To formalize the problem, we rely on a notion of {\em $\cH$-independence} introduced below. 
\begin{defn}[$\cH$-Independence]\label{def:ind_inner}
We say that $\mZ \in \mathcal{\altL}^2(\Omega,\mathcal{A},P;\cH)$ and $S: \Omega \rightarrow \mathbb S$ are $\cH$-independent if and only if for all $h \in \cH$ and all bounded measurable $g:\mathbb{S} \rightarrow \mathbb{R}$ it holds that,
$E(\langle h,\mZ \rangle \times g(S)) = \E \langle h,\mZ\rangle  \times \E(g(S)).$
\end{defn}
Thus, instead of solving for $Z: \Omega \rightarrow \mathbb X$ in \eqref{eq:obj1}, we seek a solution to the following optimization problem. 
\begin{prob}\label{prob:main}
Find $\mZ \in \mathcal{\altL}^2(\Omega,\mathcal{A},P;\cH)$ that is $\cH$-independent from $S$ (in the sense of Definition~\ref{def:ind_inner}) and is close to $X$ in the sense that
\[
\bn \mZ - \phi(X) \bn_2 \leq \bn \mZ' -\phi(X) \bn_2
\]
for all $\mZ'$ which are also $\cH$-independent of $S$.
\end{prob}
Observe that the $\cH$-independence constraint imposed by Problem~\ref{prob:main}, ensures that all non-linear predictions based on $\mZ$ are uncorrelated with the sensitive features $S$. The setting is summarized in Figure \ref{fig:dependence_bound}(a).

\paragraph{Projection onto $\mathcal{M}$.}
If $\mZ$ lies in the image of $\phi$ and $\cH$ is a `large' RKHS then $\cH$-independence also implies complete independence between the estimator $\langle \hat h, \mZ\rangle$ and $S$. To see this, assume that there exists a random variable 
$W:\Omega \rightarrow \mathbb{X}$ such that $\mZ = \phi(W)$ and that the RKHS is \textit{characteristic}.  Since for any $f\in \cH$ and bounded measurable $g:\mathbb{S}\rightarrow \mathbb{R}$
\begin{align*}
E(f(W) \times  g(S)) &= E( \langle f, \mZ \rangle \times g(S))  
= E \langle f, \mZ \rangle \cdot E(g(S))
= E( f(W) ) \cdot E( g(S))
\end{align*}
we can deduce that $W$ and $S$ is independent. Moreover, since $\mZ$ is a function of $W$  it is also independent of $S$. 
In general, $\mZ$ will not be representable as some $\phi(W)$ and there can be dependencies between $\langle \hat h, \mZ \rangle$ and $S$. However, if $\mZ$ attains values close to the manifold $\mathcal{M}$ then we can find a random variable $W$ such that $\phi(W)$ is close to $\mZ$ and the dependence between $\phi(W)$ and 
$S$ is controlled by how close $\mZ$ is to the manifold.

Showing that a suitable $W$ exists is not trivial; the difficulty is that for values that $\mZ$ might attain in $\cH$ there can be many points on the manifold closest to that value and selecting points on the manifold in a way that makes the random variable $W$ well defined needs a result on \textit{measurable selections}. The following proposition makes use of such a selection and guarantees the existence of a suitable $W$, i.e. 
it states that there exists a random variable $W$ such that $\phi(W)$ achieves the minimal distance to $\mZ + h^*$.
\begin{prop}\label{prop:multival}
Consider $\mZ \in \mathcal{\altL}^2(\Omega,\mathcal{A},P;\cH)$, assume that the kernel function is continuous and (strictly) positive-definite, and $\mathbb{X}$ is compact. For any $h^* \in \cH$ there exists a $\sigma(\mZ)$-measurable  random variable $W$ which attains values in $\mathbb{X}$ such that $\phi(W) \in \mathcal{\altL}^2(\Omega,\mathcal{A},P;\cH)$ and 
\[\bn \mZ + h^* -  \phi(W)  \bn_2 = d(\mZ + h^*,\mathcal{M}).\] 
\end{prop}
\begin{proof}
Proof is provided in Appendix \ref{app:multival}.
\end{proof}
We will call such a variable $W$ provided by the proposition \textit{a projection of $\mZ$ on $\mathcal{M}$}.
The variable $W$ can be approximated algorithmically for a given $\mZ$ and $h^*$  (see Appendix \ref{suppl:alg_W}). Furthermore, $\phi(W)$ is a good  approximation of $\phi(X)$ whenever $\mZ$ is, as
\[
\|\phi(W(\omega)) - \phi(X(\omega)) \| \leq \|\phi(W(\omega)) - \mZ(\omega)\| + \|\mZ(\omega)- \phi(X(\omega)) \| 
\leq 2 \|\mZ(\omega)- \phi(X(\omega)) \|,
\]   
where we used that $\phi(W(\omega))$ is closest to $\mZ(\omega)$ on $\mathcal{M} = \phi[\mathbb{X}]$. Therefore, 
\[
\bn \phi(W) - \phi(X) \bn_2 \leq 2 \bn \mZ- \phi(X) \bn_2.
\]

\section{Bounding the dependencies} \label{sec:bounds_on_dep}
A common approach to quantifying the dependence between random variables is to consider 
\[
|P(A\cap B) - P(A) P(B)|
\]
where $A$ and $B$ run over suitable families of events. In our setting, these families are the $\sigma$-algebras $\sigma(\mZ)$ (or, alternatively, $\sigma(W)$) and $\sigma(S)$, and the difference between $P(A\cap B)$ and $P(A) P(B)$, $A \in \sigma(\mZ)$ or $A \in \sigma(W), B\in \sigma(S)$, quantifies the dependence between the random variables $\mZ$ and $S$, and $W$ and $S$, respectively. Upper bounds on the absolute difference of these two quantities are related to the notion of $\alpha$-dependence which underlies $\alpha$-mixing. In times-series analysis mixing conditions like $\alpha$-mixing play a significant role since they provide means to control temporal dependencies (see, e.g., \citep{BRAD07,DOUK94}). The aim of this section is to show how the notion of $\cH$-independence is related to the dependence between the random variables.
In particular, Proposition~\ref{prop:alpha_dep} below states a bound on the dependence between $W$ and $S$ in terms of the distance of $\mZ$ to the manifold $\mathcal{M}$. More exactly, we allow $\mZ$ to be translated by $h^* \in \cH$ before measuring the distance. This is important because the manifold itself can lie away from the origin while  the $\mZ$ we construct in Section  \ref{sec:best_ind_feat} lies around the origin. The distance we consider is 
\[
d^2(\mZ + h^*, \mathcal{M}) = E( \inf_{h \in \mathcal{M}} \|\mZ + h^* - h\|^2),
\]
the average Hilbert space distance between $\mZ +h^*$ and the manifold. Observe that the expectation on the right side is well defined when $\mathcal{M}$ is compact since we can then replace $\mathcal{M}$ with a countable dense subset of $\mathcal{M}$.  

Furthermore, if $\mZ$ and $\phi(W)$ are closely coupled in the sense
that there exists a constant $c$ such that for any event $A_1 \in \sigma(\mZ)$ there exist an event $A_2 \in \sigma(W)$ fulfilling $P(A_1 \triangle A_2) \leq c$ then the dependence between $\mZ$ and $S$ can also be bounded. For the bound to be useful we want a small value of $c$ for which the above holds, e.g. if we let $c=1$ then the above holds trivially but the bound we provide below becomes vacuous. 
In this context, observe that $W$, as constructed above, is a function of $\mZ$ and we know that $\sigma(W) \subset \sigma(\mZ)$. However, the opposite inclusion is not guaranteed to hold.

Coming back to bounding the dependence between $W$ and $S$: the high level idea is that $\cH$-independence would correspond to normal independence if we had function evaluations `$h(Z)$' instead of inner products $\langle h,\mZ \rangle$ (given that $\cH$ is sufficient to approximate indicator functions). While generally there is no such expression for the inner product we know that for $\phi(W)$ we actually have the equivalence $\langle h,\phi(W) \rangle = h(W)$ due to the reproducing property of the kernel function. In contrast to $\mZ$  the random variable $\phi(W)$ does not need to be $\mathcal{H}$-independent  of $S$, however, if $\mZ + h^*$ and $\phi(W)$ are not too far from each other in $\|\cdot \|_2$-norm then $\phi(W)$ will be approximately $\mathcal{H}$-independent of $S$ and we can say something about the dependence between $W$ and $S$.  
Therefore, the bound below is stated in terms of $\|\mZ +h^* - \phi(W) \|_2$, which is equal to the distance between $\mZ +h^*$ and $\mathcal{M}$, and a measure of how well indicator functions can be approximated. More specifically, the bound is controlled by the functional 
\begin{equation} \label{eq:psi}
\psi(A) = \inf_{f \in \cH} 2\| \chi A(W) - f(W)\|_2 + \|f\| \, d(\mZ + h^*,\mathcal{M}),  
\end{equation}
where $A \in \{ W[C] : C \in \sigma(W)\}$ and $f$ has to balance between approximating the indicator function while keeping $\|f\| \, d(\mZ + h^*,\mathcal{M})$ small. The function $\psi$ has a natural interpretation as the minimal error that can be achieved in a regularized interpolation problem. If $\cH$ lies dense in a certain space, then any relevant indicator can in principle be approximated arbitrary well. This is not saying that $\psi(A)$ will be small since the norm of the element that approximates the indicator might be large. But the approximation error, which is $\|\chi A(W) - f(W)\|_2$, can be made arbitrary small. With this notation in place the proposition is as follows.

\begin{prop} \label{prop:alpha_dep}
Consider a $\mZ \in \mathcal{\altL}^2(\Omega,\mathcal{A},P;\cH)$ which is $\cH$-independent from $S$, suppose that the kernel function is continuous and (strictly) positive-definite, and $\mathbb{X}$ is compact.
Let $W$ be a projection of $\mZ$ on $\mathcal{M}$. For any $A\in \sigma(W)$ and $B\in \sigma(S)$, with $A' = W[A]$ being the image of $A$ under $W$, the following holds,
\[
|P(A \cap B) - P(A) P(B) | \leq \psi(A').
\]
Furthermore, for $A \in \sigma(\mZ)$, if $c>0$ is such that $\mathcal{B}_A = \{W[C] : C\in \sigma(W), P(C\triangle A) \leq c\}$ is non-empty then for any $B\in\sigma(S)$, 
\begin{align*}
&|P(A \cap B) - P(A)P(B)| \leq 2 c + \inf_{D \in \mathcal{B}_A} \psi(D).
\end{align*}  
\end{prop}
\begin{proof}
Proof is provided in Appendix \ref{app:proof_dep}.
\end{proof}
Intuitively, as visualized in Figure \ref{fig:dependence_bound}, the proposition states that if $\mZ$ mostly attains values in the gray area then the dependence between $W$ and $S$ is low and, if $W$ and $\mZ$ is strongly coupled, then the dependence between 
$\mZ$ and $S$ is also low.

\subsection{Estimating $\psi(A)$}
The key quantity  in Proposition \ref{prop:alpha_dep} is $\psi(A)$. To control $\psi(A)$ it is necessary to control how well the RKHS can approximate indicators and to estimate the distance $d(\mZ + h^*,\mathcal{M})$. The former problem is more difficult and might be approached using the theory of interpolation spaces; we do not try to develop the necessary theory here but only mention a simple result on denseness at the end of this section. On the other hand, the latter problem is easy to deal with:  
the distance $d(\mZ + h^*,\mathcal{M})$ between $\mZ + h^*$ and $\mathcal{M}$ can be estimated efficiently.
In the case where the space $\mathbb X$ is compact and $\phi$ is a continuous function, we propose an empirical estimate of $d(\mZ + h^*,\mathcal{M})$ given by
\begin{equation}\label{eq:dnestim}
d_n(\mZ+h^*,\cM):=\frac{1}{n}\sum_{i=1}^n \min_{h \in \cM} \|\mZ_i+h^*-h\|
\end{equation}
where ${\mZ}_i,~i  \leq n,~n \in \mathbb N $, are $n$ independent copies of $\mZ$.
Note that the compactness of $\mathbb X$ together with the continuity of $\phi$ make the $\min$ operator in \eqref{eq:dnestim} well-defined.  
\begin{prop}\label{prop:dist_chaining}
Consider a $\mZ \in \mathcal{\altL}^2(\Omega,\mathcal{A},P;\cH)$ which is $\cH$-independent from $S$, suppose that the kernel function is continuous and (strictly) positive-definite, and $\mathbb{X}$ is compact. 
Let $\rho = \max_{x \in \mathbb{X}} \|\phi(x)\| < \infty$. For any $h^* \in \cH$ with $\|h^*\| \leq \rho$ and every $\epsilon >0$ we have,
\begin{equation*}
\Pr(|d_n(\mZ+h^*,\cM) - d(\mZ+h^*,\cM)| \geq \epsilon) \leq  2\exp\Bigl(-\frac{2n \epsilon^2}{25 \rho^2}\Bigr).
\end{equation*}
\end{prop} 
\begin{proof}
Proof is provided in Appendix~\ref{app:dist_chaining}. 
\end{proof}
Coming back to the approximation error 
$\| \chi A(W) - f(W)\|_2$,  where $A \subset \mathbb{X}$ is the image under $W$ of some set $C\in \sigma(W)$ and $f\in \cH$ we like to mention the following: 
let $\nu = P W^{-1}$ be the push-forward measure of $P$ under $W$.
If $\cH$ lies dense in $\mathcal{\altL}^2(\mathbb{X},\mathcal{B},\nu)$ then for any such $A$ and any $\epsilon>0$ there exists a function $f$ such that $\| \chi A(W) - f(W)\|_2 < \epsilon$, i.e. for the measurable set $A$ there exists a function $f \in \cH$ such that 
\begin{align*}  
\int (\chi A(W) - f(W))^2 \,dP = \int (\chi A(x) - f(x))^2 \, d\nu(x) < \epsilon^2,  
\end{align*}
using \cite[Theorem 235Gb]{FREM}. In many cases the continuous functions $C(\mathbb{X})$ lie dense in $\mathcal{\altL}^2(\mathbb{X},\mathcal{B},\nu)$ and a universal RKHS $\cH$ is sufficient to approximate the indicators $\chi W$ (see \cite{BHAR11}).

\section{Best $\cH$-independent features} \label{sec:best_ind_feat}
In this section we discuss how to obtain $\mZ$ as a closed-form solution to Problem~\ref{prob:main}. To this end, 
inspired by the sub-problem in the linear case, we obtain $\mZ$ using Hilbert-space-valued conditional expectations. We further show that these features are $\cH$-independent of $S$ and that $\mZ$ is the best $\cH$-independent approximation of $\phi(X)$.

In the linear case discussed in the Introduction it turned out that $Z = X - \E^S X  + E X$ is a good candidate for the new features $Z$. In the Hilbert-space-valued case a similar result holds. The main difference here is that we do have to work with Hilbert-space-valued conditional expectations. 
For any random variable $\mX\in \mathcal{\altL}^2(\Omega,\mathcal{A},P;\cH)$, and any
$\sigma$-subalgebra $\mathcal{B}$ of $\mathcal{A}$, conditional expectation $E^\mathcal{B} \mX$ is defined and is again an element of $\mathcal{\altL}^2(\Omega,\mathcal{A},P;\cH)$. 
We are particularly interested in conditioning with respect to the sensitive random variable $S$. In this case, $\mathcal{B}$ is chosen as $\sigma(S)$, the smallest $\sigma$-subalgebra which makes $S$ measurable, and we denote this conditional expectation by $E^S \mX$. In the following, we  use the notation $\mX = \phi(X) $. A natural choice for the new features is 
\begin{equation} \label{eq:new_feature_def}
\mZ = \mX - \E^S \mX + E(\mX).
\end{equation}
The expectation $E(\mX)$ is to be interpreted as the Bochner-integral of $\mX$ given measure $P$. Importantly, if $S$ and $\mX$ are independent, we have with this choice that $\mZ = \mX = \phi(X)$ and we are back to the standard kernel setting. Also, if $\mX \in \mathcal{\altL}^2(\Omega,\mathcal{A},P;\cH)$ then so is $\mZ$.

We can verify that the features $\mZ$ are, in fact, $\cH$-independent of $S$.  In particular, for any $h \in \cH$ and $g\in \mathcal{L}^2$,
\begin{align*}
&E (\langle \mX - E^S \mX, h\rangle \times g(S) ) \\
&= \langle E(  \mX \times g(S) ) - E( (E^S \mX) \times g(S)), h \rangle  \\
&= \langle E(  \mX \times  g(S) ) - E( E^S (\mX \times g(S)) ), h \rangle =0.
\end{align*}
Since $E(\mX)$ is a constant this implies that
$E(\langle \mZ,h \rangle \times g(S)) = E(h(X)) \cdot E(g(S))$ A similar argument shows that $E\langle \mZ,h \rangle =E(h(X))$. Thus, $\mZ$ is $\cH$-independent of $S$.

\begin{figure*}[t]
\begin{center}
\includegraphics[scale=0.3]{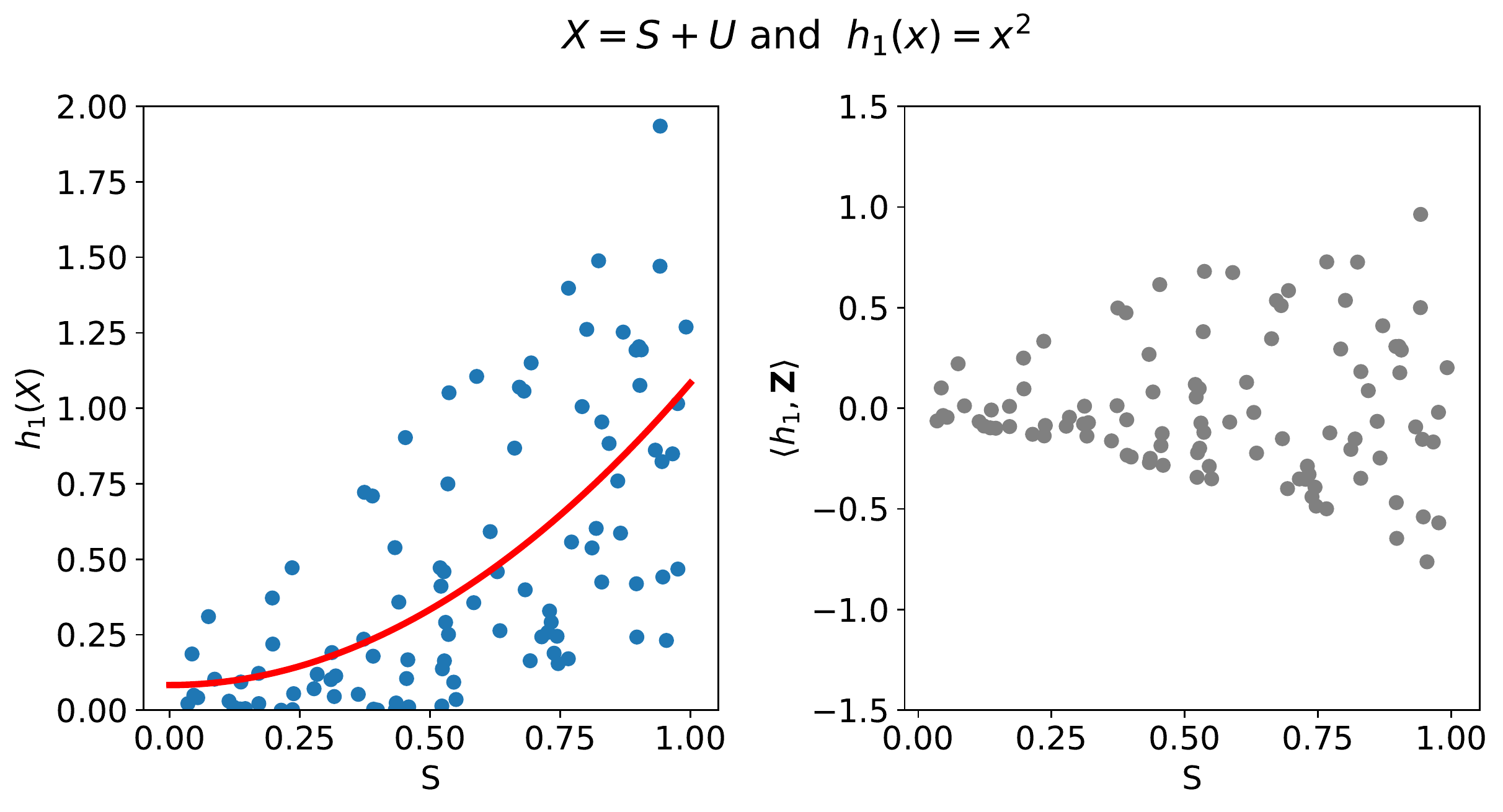}
\hspace{0.1cm}
\includegraphics[scale=0.3]{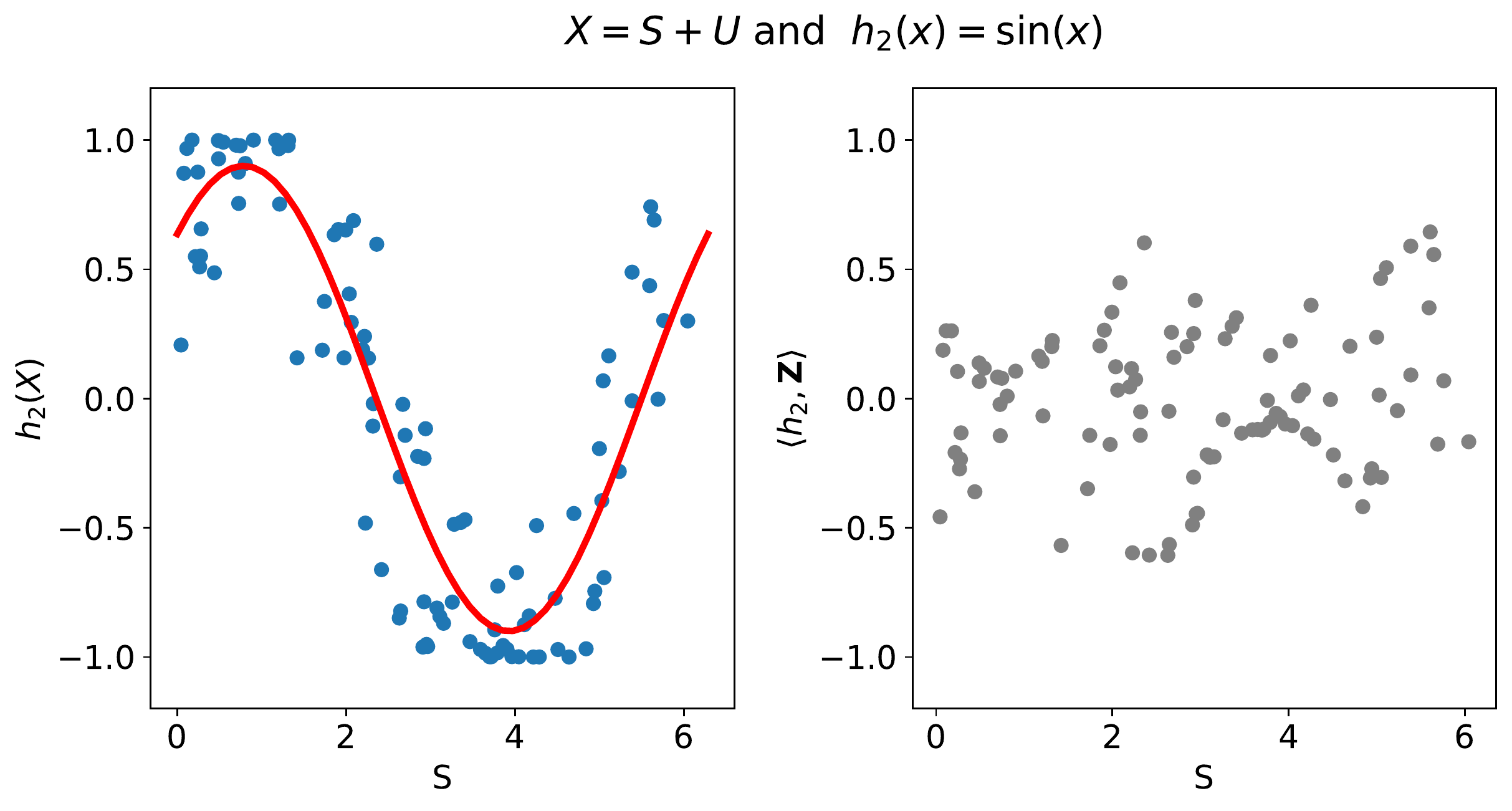}
\end{center}
\caption{The figure shows data from two different settings. In the left two plots $X = S+ U$, where $S$ and $U$ are independent, $S$ is  uniformly distributed on $[0,1]$ and $U$ is uniformly distributed on $[-1/2,1/2]$. The function $h_1$ is the quadratic function. The leftmost plot shows $h_1(X)$ against $S$ and the plot to its right shows a centered version of $\langle h_1,\mZ \rangle$  plotted against $S$. Similarly for the right plots with the difference that $S$ is uniformly distributed on $[0,2\pi]$ and $U$ is uniformly distributed on  $[0,\pi/2]$. The function $h_2(x)$ is $\sin(x)$. The red curves show the best regression curve, predicting $h_1(X)$ and $h_2(X)$ using $S$.}
\label{fig:uncorr}
\end{figure*}

In Figure \ref{fig:uncorr} the effect of the move from $\mX$ to $\mZ$ is visualized. In the figure $S$ is plotted against $h_1(X)$ and $h_2(X)$ (blue dots), where $h_1$ corresponds to the quadratic function and $h_2$ to the sinus function. The dependencies between $h_1(X)$ and $S$, as well as $h_2(X)$ and $S$, are high and there is clear trend in the data. The two red curves correspond  to the best regression functions, using $S$ to predict $h_1(X)$ and $h_2(X)$. The relation between the new features and $S$ is shown in the other two plots (gray dots). In the case of $h_1$ one can observe that the dependence between $\langle h_1, \mZ \rangle$ and $S$ is much smaller and, by the design of $\mZ$, $\langle h_1, \mZ \rangle$ and $S$ are uncorrelated. Similarly, for $\langle h_2,\mZ\rangle$, whereas here the dependence to $S$ seems to be even lower and it is difficult to visually verify any remaining dependence between $S$ and $\langle h_2,\mZ\rangle$.

An interesting aspect of this transformation from $X$ to $\mZ$ is that $\mZ$ is automatically uncorrelated with $S$ for all functions $h$ in the corresponding RKHS, without the need to ever explicitly consider a  particular $h$. 
Besides being $\cH$-independent of $S$ these new features $\mZ$ also closely approximates our original features $\mX$ if the influence from $S$ is not too strong, i.e. the mean squared distance is 
\[
 E(\| \mX - \mZ\|^2 ) = \E(\| \E^S \mX - E(\mX) \|^2)
\]
which is equal to zero if $X$ is independent of $S$. In fact, $\mZ$ is the best approximation of $\mX$ in the mean squared sense under the $\cH$-independent constraint. This is  essentially a property of the conditional expectation which corresponds to an orthogonal projection in $L^2(\Omega,\mathcal{A},P;\cH)$. We summarize this property in the following result.
\begin{prop} \label{lem:best_approx}
Given $\mX,\mZ' \in \mathcal{\altL}^2(\Omega,\mathcal{A}, P;\cH)$ such that $\mZ'$ is $\cH$-independent of $S$, then  
\[
E(\|\mX - \mZ'\|^2)  \geq E( \| \mX - \mZ \|^2 ),
\]
where $\mZ = \mX - E^S \mX + E(\mX)$. Furthermore, $\mZ$ is the unique minimizer (up to almost sure equivalence). 
\end{prop}
\begin{proof} 
Proof provided in  Appendix \ref{app:best_approx}.
\end{proof}

\paragraph{Change in predictions.}
When replacing $\mX$ by $\mZ$ we lose information (we reduce the influence of the sensitive features). An interesting question to ask is, `how much does the reduction in information change our predictions?' A simple way to bound the difference in predictions is as follows. Consider any $h \in \cH$, for instance corresponding to a regression function, then  
\begin{align*}
    &| h(X) - \langle h, \mZ \rangle | 
    \leq \|h\| \|\mX - \mZ\| 
    \leq \|h\| \| \E^S \mX  - E(\mX) \|
\end{align*}
where $\|E^S\mX - E(\mX)\|$ effectively measures the influence of $S$. Hence, the difference in prediction is upper bound by the norm of the predictor (here $h$) and a quantity that measures the dependence between $S$ and $\mX$.

\paragraph{Example.}
To demonstrate that the effect of the move from $X$ to $\mZ$ can be profound we consider the following 
fundamental example:  suppose that $X$ and $S$ are standard normal random variables with covariance $c\in[-1,1]$ and consider the \textit{linear kernel} 
$k(x,y)  = xy$, $x,y\in \mathbb{R}$. In this case
$\phi(X) = X$ and $E^S X = c S$  is also normally distributed (see  \citet{BERT01}[Sec4.7]). Hence, $\mZ = X - E^S X+ E(X)$ is normally distributed and 
$E(\mZ \times S) = c - c E(S^2) = 0$. This implies that $\mZ$ and $S$ are, in fact, \textit{fully independent}, regardless of how large the dependence between the original features $X$ and the sensitive features $S$ may be. In the case where $X$ and $S$ are fully dependent, i.e. $X=aS$ for some $a\in \mathbb{R}$, the features $\mZ$ are equal to zero   and do not approximate $X$.

Next, consider a \textit{polynomial kernel} of second order such that the quadratic function $h(x) = x^2$ lies within the corresponding RKHS. The inner product between this $h$ and $\mZ$ is equal to $X^2 - E^S X^2 +E(X^2)$  and \textit{is not independent} of $S$. Hence, the \textit{kernel function affects the dependence} between $\mZ$ and $S$. Also, within the same RKHS there lie linear functions  and for any linear function $h'$ it holds that $\langle \mZ, h' \rangle$ is independent of $S$. Therefore, within the same RKHS we can have directions in which $\mZ$ is independent of $S$ and directions where both variables are dependent. 

\section{Generating oblivious features from data} \label{sec:est_cond_exp}
To be able to generate the features $\mZ$ we need to first estimate the conditional expectation $E^S \phi(X)$ from data. To this end, we devise a plugin-approach. After introducing this approach in Section \ref{sec:plugin_est} we discuss how the estimation errors of the plugin-estimator can be controlled in Section \ref{sec:est_error}. In Section \ref{sec:gen_features} we show how the oblivious features can be generated. Finally, in Section \ref{sec:obl_reg}, we demonstrate how the approach can be applied to statistical problems and we discuss relations to the approach of \citet{DON18}  in Section \ref{sec:disc_don}.

\subsection{Plug-in estimator} \label{sec:plugin_est}
A common method for estimation is the plug-in approach whereby an unknown probability measure is replaced by the empirical measure. This approach is used in \citet{SG} for deriving estimators of conditional expectations. To see how the approach can be generalized to our setting, first observe that we can write 
\begin{equation} \label{eq:cond_exp_g}
E^S \mX = g\circ S \textit{\quad almost surely},
\end{equation}
where  $g:\mathbb{S} \rightarrow \cH$ is a Bochner-measurable function (see Appendix \ref{app:prob_in_H} and Lemma \ref{lem:cond_exp_g} for details). Our aim is to estimate this function $g$ from i.i.d. observations $\{(X_i,S_i)\}_{i\leq n}$.
For any subset $B$ of the range space $\mathbb{S}$ of the sensitive features define the empirical measure 
$
P_n(S\in B) = (1/n) \sum_{i=1}^n \delta_{S_i}(B),
$
where $\delta_{S_i}$ the Dirac measure with mass one at location $S_i$. We define  an estimate of the conditional expectation of $\mX$ given that the sensitive variable falls into a set $B$ by
\[
E_n(\mX |S \in B) = \frac{1}{n P_n(S \in B)} \sum_{i=1}^n \phi(X_i) \times \delta_{S_i}(B), 
\]
when $P_n(S\in B) > 0$ and through $E_n(\mX|S \in B) =0$ otherwise. Observe that for $h\in \cH$ we have,
\begin{align*}
&\bigl\langle h,  \frac{1}{n P_n(S \in B)} \sum_{i=1}^n \phi(X_i) \times \delta_{S_i}(B) \bigr\rangle  
= \frac{1}{n P_n(S \in B)} \sum_{i=1}^n h(X_i) \times \delta_{S_i}(B).
\end{align*}
We can also write this as $\langle h, E_n(\mX|S\in B)\rangle = E_n(h(X)|S\in B)$.
An estimate of the conditional expectation given $S$ is provided by
\[
E^S_n \mX = \sum_{B\in \wp_S}  E_n(\mX|S\in B) \times \chi \{ S\in B\},
\]
where $\wp_S$ is a finite partition of the range space $\mathbb{S}$ of $S$. A common choice for $\wp_S$ if $\mathbb{S}$ is the hypercube $[0,1]^d$, $d\geq 1$, are the dyadic sets.
Observe, that we can move inner products inside the conditional expectation $E_n^S \mX$ so that
$\langle h,E_n^S \mX  \rangle = E_n^S h(X)$, where $E_n^S h(X)$ is the empirical conditional expectation introduced in \cite{SG}.

\subsection{Controlling the estimation error} 
\label{sec:est_error}
The estimation error when estimating $E^S \mX$ using $E_n^S \mX$ is relatively easy to control thanks to the plug-in approach. Essentially, standard results concerning the empirical measure carry over to  conditional expectation estimates in the  real-valued case \citep{SG}. But through scalarization we can transfer some of  these results straight away to the Hilbert-space-valued case.  For instance, using $\phi(X)$ in place of $\mX$,
\begin{align*} 
&\| E_n(
\phi(X) | S \in B) - E(\phi(X) | S\in B)\|  \\
 &= \sup_{\|h\| \leq 1} |\langle E_n(\phi(X) | S \in B) - E(\phi(X) | S\in B), h \rangle| \\
 &= \sup_{\|h\| \leq 1}  |E_n(h(X) | S \in B) - E(h(X) | S\in B)|
\end{align*}
and bounds on the latter term are known.
Similarly, 
\begin{align}
\| E_n^S \phi(X) - E^S \phi(X)\|  
&= \sup_{\|h\|\leq 1} |E_n^S h(X) - E^S(h(X))|. \label{eq:uni_bound} 
\end{align}
However, both $E_n^S \phi(X)$ and $E^S \phi(X)$ are random variables and a useful measure of their difference is the $\mathcal{\altL}^2$-pseudo-norm. The  $\mathcal{\altL}^2$-pseudo-norm should in this case not be taken with respect to $P$ itself but conditional on the training sample. Hence, for i.i.d. pairs $(X,S), (X_1,S_1),\ldots, (X_n,S_n)$ let $\mathcal{F}_n = \sigma(X_1,S_1,\ldots, X_n,S_n)$ and define the `conditional' $\mathcal{\altL}^2$-pseudo-norm by
\[
\bn E_n^S \phi(X) - E^S \phi(X) \bn^2_{2,n} = E^{\mathcal{F}_n} \|E_n^S \phi(X) - E^S \phi(X)\|^2.
\]
Substituting Equation \eqref{eq:uni_bound} in shows that this expression is equal to
\[
E^{\mathcal{F}_n} \Bigl(\, \sup_{\|h\| \leq 1} |E_n^S h(X) - E^S h(X)|^2\Bigr).
\]
The supremum cannot be taken out of the conditional expectation, however, by writing $E_n^S h(X)$ and  $E^S h(X)$ as simple functions (see Appendix \ref{subsec:mesfun}) we can get around this difficulty and control the error in $\bn \cdot \bn_{2,n}$. We demonstrate this in the following by deriving rates of convergence for two cases: for the case where $\mathbb{S}$ is finite, and for the case where $\mathbb{S}$ is the unit cube in $\mathbb{R}^d$ for some $d \geq 1$ and $S$ has a density that is bounded away from zero.  

To derive these rates we rely, among other things, on the  convergence of the empirical process uniformly over families of functions related to the unit ball of $\cH$ and partitions of $\mathbb{S}$. For instance, in the case where $\mathbb{S}$ is finite we need to assume that 
\[
\mathcal{H}_{\mathbb{S}} :=  \{(h \circ \pi_1) \times \chi(\mathbb{X} \times \{s\})  : h \in \cH, \|h\| \leq 1, s\in \mathbb{S}\},
\]
as a family of real-valued functions  on $\mathbb{X} \times \mathbb{S}$, is a $P$-Donsker class. The function $\pi_1:\mathbb{X} \times \mathbb{S} \rightarrow \mathbb{X}$ is here the projection onto the first argument, i.e. $\pi_1(x,s) =x$. For the definition of \textit{$P$-Donsker classes} see \cite{DUDLEY14,GINE16}.

There are various ways to verify this condition in concrete settings. For example, if $\cH$ is a finite dimensional RKHS then $\cH_\mathbb{S}$ is a $P$-Donsker class under a mild measurbility assumption. This follows from a few simple arguments: any finite dimensional space of functions is a VC-subgraph class \citep[Ex.3.6.11]{GINE16}; this implies directly that  $\{(h \circ \pi_1) \times \chi (\mathbb{X} \times \{s\}) : h \in \cH, \|h \| \leq 1 \}$ is a VC-subgraph class for every $s \in \mathbb{S}$. Furthermore, finite unions of VC-subgraph classes are again a VC-subgraph class; under a mild measurbility assumption it follows now from \citet[Cor.6.19]{DUDLEY14}  that $\cH_\mathbb{S}$ is a $P$-Donsker class.

There are obviously other ways to prove this statement. In particular, one might use that 
the unit ball of $\cH$ is a \textit{universal Donsker class} (see \cite{DUDLEY14,GINE16} for details) when the kernel function is continuous and $\mathbb{X}$ is compact (this also holds when $\cH$ is infinite dimensional):  due to \cite{Marcus1985} the unit ball of a Hilbert space is a universal Donsker class if $\sup_{x\in \mathbb{X}} |h(x)| \leq c \|h\|$ for some constant $c$ that does not depend on $h$.
If the kernel function is bounded $c = \sqrt{k(x,x)}$ witnesses that this property holds.

\paragraph{Case 1: finitely many sensitive features.}
Our first proposition states that the estimator converges with the optimal rate $n^{-1/2}$ when $\mathbb{S}$ is finite and $\cH_\mathbb{S}$ is a $P$-Donsker class.
\begin{prop} \label{Prop:Est_error_cond}
Given a finite space $\mathbb{S}$ and a $P$-Donsker class $\cH_\mathbb{S}$, it holds that 
\[
\bn E_n^S \phi(X) - E^S \phi(X) \bn_{2,n} \in O^*_{P}(n^{-1/2}).
\]
\end{prop}
\begin{proof}
The proof is given in Appendix~\ref{app:proof_est_error}.
\end{proof}
\paragraph{Case 2: $[0,1]^d$-valued sensitive features.} We extend Proposition~\ref{Prop:Est_error_cond} to the case where $S$ is not confined to taking finitely many values. In order to state the result, we introduce the following notation. Set $\mathbb S:=[0,1]^d$ for some $d \in \mathbb N$ and let $g:\mathbb{S} \rightarrow \cH$ be such that with probability one $E^S \mX = g\circ S$ (which is possible by Lemma~\ref{lem:cond_exp_g}). Consider a discretization of $\mathbb S$ into dyadic cubes $\Delta_1,\Delta_2,\dots,\Delta_{{\ell^d}}$ of side-length $1/{\ell}$ for some $\ell \in \mathbb N$. Define $\mathfrak{C}_{\ell}:=\{\mathbb X \times \Delta: \Delta \in \mathfrak{D}_{\ell}\}$ and let $\cH_{\mathfrak{C}}:=\{h \times \chi {D}: h\in \cH,~\|h\|\leq 1,~D \in \bigcup_{\ell \in \mathbb N} \mathfrak{C}_{\ell}\}$. 
\begin{prop} \label{Prop:Est_error_cond_cont}
Suppose that the push forward measure $\mu:=P S^{-1}$  has density $u$ with respect to the Lebesgue measure $\lambda$ on $\mathbb S$ with the property that $\inf_{s \in \mathbb S} u(s) \geq b$ for some $b >0$. Assume that $g\circ S$ is $L$-lipschitz continuous and that $\cH_{\mathfrak{C}}$ is a $P$-Donsker class.
We have
\[
\bn E_n^S \phi(X) - E^S \phi(X) \bn_{2,n} \in O^*_{P}(n^{-\frac{1}{2(d+1)}}).
\]
\end{prop}
\begin{proof}
The proof is given in Appendix~\ref{app:proof_est_error_cont}.
\end{proof}
\subsection{Generating an oblivious random variable} \label{sec:gen_features}
Given a data-point $(X,S)$ composed of non-sensitive and sensitive features $X$ and $S$ respectively, we can generate an {\em oblivious} random variable $\mZ$ as
\begin{equation}\label{eq:Zempcalc}
\mZ:= \phi(X) - E_n^{S} \phi(X) + E_n(\phi (X)).
\end{equation}
Most kernel methods work with the kernel matrix and do not need access to the features themselves. The same holds in our setting. More specifically, we never need to represent $\mZ$ explicitly in the Hilbert space but only require inner-product calculations. 
In order to calculate the empirical estimates of the conditional expectation $E^S_n \mX$ and of $E_n( \mX)$ in \eqref{eq:Zempcalc}
we consider a simple approach whereby we split the training set into two subsets of size $n$, and use half the observations to obtain the empirical estimates of the expectations. 
The remaining $n$ observations are used to obtain an {\em oblivious} predictor; we have two cases as follows. 
\paragraph{Case 1 (M-Oblivious).}  The standard kernel matrix $K$ is calculated with the remaining $n$ observations and a kernel-method is applied to $K$ to obtain a predictor $g$.
When applying the predictor to a new unseen data-point $(X,S)$ we first transform $X$ into $\mZ$ via \eqref{eq:Zempcalc} and calculate the prediction as $\langle g, \mZ \rangle$. As discussed in the Introduction, we conjecture that this approach is suitable in the case
where the labels $Y$ are conditionally independent of the sensitive features $S$ given the non-sensitive features $X$, i.e. when  $S,X,Y$  form a Markov chain $S\rightarrow X \rightarrow Y$. As such we call this approach $M$-Oblivious. 

\paragraph{Case 2 (Oblivious).} Instead of calculating the kernel matrix $K$ an \textit{oblivious kernel matrix}, i.e.
\begin{align}\label{eq:oo}
\cO = \begin{pmatrix}
\| \mZ_1\|^2  & \cdots & \langle \mZ_1, \mZ_n \rangle \\
\vdots  & \ddots & \vdots\\
 \langle \mZ_n, \mZ_1\rangle & \cdots & \|\mZ_n\|^2
\end{pmatrix},
\end{align}
is calculated by applying Equation  \eqref{eq:Zempcalc} to the remaining training samples $(X_i,S_i)$ before taking inner products.
The oblivious matrix is then passed to the kernel-method to gain a predictor $g$.
The matrix is positive semi-definite since
$
\mb a^\top \cO \mb a = \| \sum_{i=1}^n a_i \mZ_i\|^2 \geq 0, 
$
for any $\mb a \in \mathbb{R}^n$.
The complexity to compute the matrix is $O(n^2)$ (see  Appendix \ref{suppl:alg} for details on the algorithm). Prediction for a new unseen data-point $(X,S)$ is now done in the same way as in Case 1.

\subsection{Oblivious ridge regression} \label{sec:obl_reg}
In this section we showcase our approach in the context of kernel ridge regression. We have three relevant random variables, namely, the non-sensitive features $X$, the sensitive features $S$ and labels $Y$ which are real valued. We assume that we have $2n$ i.i.d. observations $\{(X_i,S_i,Y_i)\}_{i\leq 2n}$. We use the observations $n+1, \ldots, 2n$ to generate the oblivious random variables $\mb Z_i$ and then use  oblivious data $\{(\mb Z_i,Y_i)\}_{i\leq n}$ for oblivious ridge regression (ORR).

The ORR problem has the following form. Given a  positive definite kernel function $k: \mathbb{X} \times \mathbb{X}\rightarrow \mathbb{R}$, a corresponding RKHS $\cH$ and oblivious features $\mb Z_i$. Our aim is to find a regression function $h \in \cH$ such that the mean squared error between $\langle h, \mZ \rangle $ and $Y$ 
is small. Replacing the mean squared error by the empirical least-squares error and adding a regularization term for $h$ gives us the optimization problem
\begin{align}\label{eq:lse}
\hat h = \argmin_{h\in \cH} \sum_{i=1}^n (\langle h, \mb Z_i \rangle - Y_i)^2 + \lambda \|h\|^2,
\end{align}
where $\lambda>0$ is the regularization parameter.

It is easy to see that the setting is not substantially different from standard kernel ridge regression and derive a closed form solution for $\hat h$. More specifically, we have a representer theorem in this setting which tells us that the minimizer lies in the span of $\mb Z_1,\ldots, \mb Z_n$.  One can then solve the optimization problem in the same way as for standard kernel ridge regression, see Appendix \ref{app:ridge_reg_opt} for details. The solution to the optimization problem is
$\hat h = \sum_{i=1}^n \alpha_i \mZ_i$, where $\bm \alpha = (\mathcal{O} + \lambda I)^{-1} \bm y$.
The vector $\bm y$ is given by $(Y_1, \ldots, Y_n)^\top$.  
Predicting $Y$ for a new observation $(X,S)$ is  achieved by first generating the oblivious features $\mb Z$  (see Appendix \ref{suppl:alg_obl_pred}) and then by evaluating
$\langle \mb Z, \hat h \rangle = \sum_{i=1}^n \alpha_i \langle \mb Z,\mb Z_i \rangle.$

\subsection{Comparison to \citep*{DON18}} \label{sec:disc_don}
Our focus in this paper is on generating features that are less dependent on the sensitive features than the original non-sensitive features. However, the conditional expectation $E^S \phi(X)$, which is at the heart of our approach, also features prominently in methods that add constraints to SVM classifiers. In particular, in \cite{DON18} a constraint is used to achieve approximately equal opportunity in classification where the sensitive feature is binary. While their approach does not make explicit use of conditional expectations one can recognize that the key object in their approach (Eq. (13) in \cite{DON18}) is, in fact, closely related to our conditional expectation when used in the case where $S$ can attain only two values (say $\mathbb{S} = \{0,1\}$). In detail, the optimization problem (14) is constraint by enforcing for a given $\epsilon >0$ that the solution $h^* \in \cH$ fulfills
\begin{equation} \label{eq:Ferm_constraint}
|E_n(h^*(X) | S= 0) - E_n(h^*(X)| S= 1)| \leq \epsilon. 
\end{equation}
Considering $\mZ = \phi(X) - E^S \phi(X) + E(\phi(X))$ 
we can observe right away that in this setting for all $h \in \cH$,
\[
E( \langle h,\mZ \rangle |S = 0)  = E( \langle h,\mZ \rangle |S = 1).
\]
To see this observe that $E^S \langle h, \mZ\rangle$ is almost surely equal to the $E(h(X))$. In other words
\[
E(\langle h,\mZ \rangle |S = 0 ) \times \chi\{S=0\} + E(\langle h,\mZ \rangle |S = 1) \times \chi\{S=1\} = 
E^S \langle h,\mZ \rangle  
\]
is almost surely constant. Unless $P(S=0)=0$ or $P(S=1)=0$ this implies that  $E(\langle h,\mZ \rangle |S = 0) = E(\langle h,\mZ \rangle |S = 1)$. Hence, for the max-margin classifier $h^*$
and $\mZ$ it holds that $E(\langle h^*,\mZ \rangle |S = 0) = E(\langle h^*,\mZ \rangle |S = 1)$ and on the population level our new features $\mZ$ guarantee that constraint \eqref{eq:Ferm_constraint} is automatically fulfilled.

\section{Empirical evaluation} \label{sec:experiments}
\begin{figure}
\centering
\begin{subfigure}{.5\textwidth}
  \centering
 \includegraphics[scale=0.5]{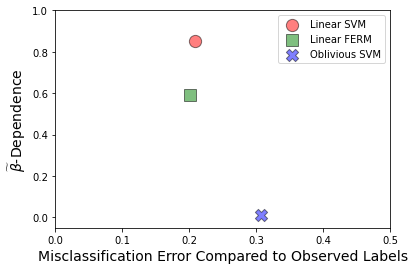}
  \caption{}\label{fig:sub1}
\end{subfigure}%
\begin{subfigure}{.5\textwidth}
  \centering
\includegraphics[scale=0.5]{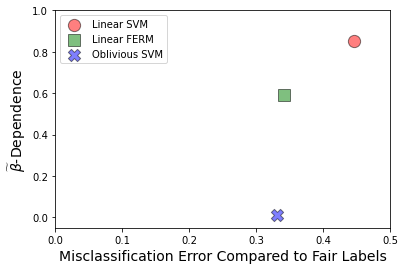}
  \caption{}\label{fig:sub2}
\end{subfigure}
\caption{Binary classification error vs. $\widetilde{\beta}$-dependence between prediction and sensitive features is shown for three different methods: classical Linear SVM, Linear FERM, and Oblivious SVM. In Figure~\ref{fig:sub1} the error is calculated with respect to the  observed labels which are intrinsically biased and in Figure~\ref{fig:sub1} the error is calculated with respect to the true {\em fair} classification rule.}
\label{fig:MC}
\end{figure}

In this section we report our experimental results for classification and regression. Our objective in the classification experiment is to point out an important property of supervised learning problems where sensitive features affect both the non-sensitive features and the labels: the estimation error of the observed labels  
can be misleading as a quality measure. The aim is much rather to predict values in an unbiased fashion. The first experiment  highlights this difference by considering a synthetic data set for which we know the unbiased labels  (though the true unbiased labels arre not available to the methods). We measure the dependencies
between the predicted values and the sensitive features, and compare against a standard SVM and to FERM. 
The second set of experiments aims to investigate how dependencies between sensitive and non-sensitive features affect ORR and M-ORR. We are investigating this relationship by considering a family of synthetic problems for which we can adjust the dependency between the features using a parameter $\gamma$. In this set of experiments we are also concerned with clarifying the relationship between ORR and M-ORR, where the latter is the M-Oblivious version of KRR, see Section~\ref{sec:gen_features}. Our implementation can be found at the following repository: \url{https://github.com/azalk/Oblivious.git}.

\subsection{Binary Classification}
We carried out an experiment to mimic a scenario where a class of students should normally receive grades between $0$ and $5$, and anyone with a grade above a fixed threshold $\theta=2$ should pass.  Half of the class, representing a ``minority group'', are disadvantaged  in that their grades are almost systematically reduced, while the other half receive a boost on average. More specifically, 
let the sensitive feature $S$ be a $\{0,1\}$-valued Bernoulli random variable with parameter $0.5$, and let $X_0$ be distributed according to a truncated normal distribution with support $[1,4]$. 
Let the non-sensitive feature $X$, representing a student's grade, be given by
\begin{equation*}
X:=(X_0-B)\chi\{S=0\}+(X_0+B)\chi\{S=1\}
\end{equation*} 
where $B$ is a Bernoulli random variable with parameter $0.9$ independent of $X_0$ and of $S$. 
The label $Y$ is defined as a noisy decision influenced by the student's ``original grade'' $X_0$ prior to the $S$-based modification. 
More formally, let $U$ be a random variable independent of $X_0$ and of $S$, and uniformly distributed on $[0,1]$. 
Let $Y_0 :=\chi\{ U \geq X_0\}$ and define
\begin{equation*}
Y:=Y_0\chi\{X+ S \geq \theta\}.
\end{equation*}
\paragraph{Classification Error.} 
In a typical classification problem,  the labels $Y$ depend on both $X$ and $S$ so when we remove the bias it is not clear what we should compare against when calculating the classification performance. 
Observe that our experimental construction here allows access to the true ground-truth labels
\begin{equation}\label{eq:Y*}
Y^*:=\chi\{X_0 \geq \theta\}.
\end{equation}
Therefore, we are able to calculate the true (unbiased) errors as well. However, this is not always the case in practice. In fact, we argue that the question of how to evaluate fair classification performance is an important open problem which has yet to be addressed.

\paragraph{Measure of Dependence.}
Let $\cF_n:=\sigma(X_1,\dots,X_n,S_1,\dots,S_n),~n \in \mathbb N$ be the $\sigma$-algebra generated by the training samples. 
In this experiment, we measure the dependence between the predicted labels $\widehat{Y}$ produced by any algorithm and the sensitive features $S$ as 
\begin{equation}\label{eq:beta_dep} 
\widetilde{\beta}(\widehat{Y},S) := \frac{1}{2}\sum_{s \in \{0,1\}}\sum_{y \in \{0,1\}}\E \left |P(\widehat{Y}=y,S=s|\cF_n)-P(\widehat{Y}=y|\cF_n)P(S=s)\right|
\end{equation}
which is closely related to the 
$\beta$-dependence (see, e.g. \citep[vol. I, p. 67]{BRAD07}) between their respective $\sigma$-algebras. 
We obtain an empirical estimate of $\widetilde{\beta}(\sigma(\widehat{Y}),\sigma(S))$ by simply replacing the  probabilities in  \eqref{eq:beta_dep} with corresponding empirical frequencies. 
\paragraph{Experimental results.}
We generated $n=1000$ training and test samples as described above and the errors reported for each experiment are averaged over $10$ repetitions. 
Figure~\ref{fig:MC} shows binary classification error vs. dependence between prediction and sensitive features for three different methods: classical Linear SVM, Linear FERM, and Oblivious SVM. In Figure~\ref{fig:sub1} the error is calculated with respect to the  observed labels which are intrinsically biased and in Figure~\ref{fig:sub1} the error is calculated with respect to the true {\em fair} classification rule $Y^*$ given by \eqref{eq:Y*}. As can be seen in the plots, the {\em true} classification error of Oblivious SVM is smaller than that of the other two methods. Moreover, in both plots the $\beta$-dependence between the predicted labels produced by Oblivious SVM and the sensitive feature is close to $0$ and is much smaller than that of the other two methods.

\begin{figure}
\centering
\begin{subfigure}{.5\textwidth}
  \centering
\includegraphics[scale=0.49]{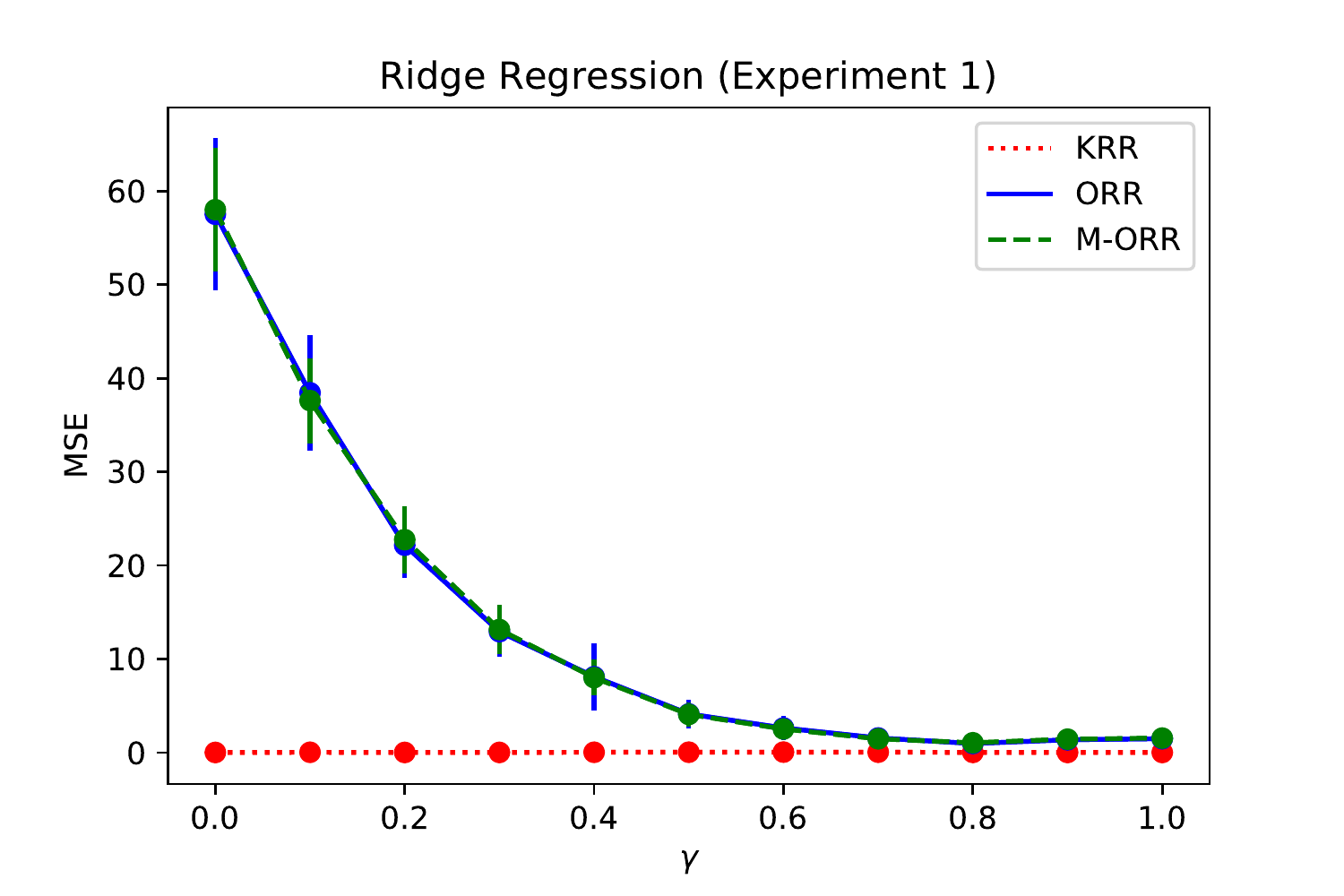}
  \caption{}\label{fig:sub1r}
\end{subfigure}%
\begin{subfigure}{.5\textwidth}
  \centering
\includegraphics[scale=0.49]{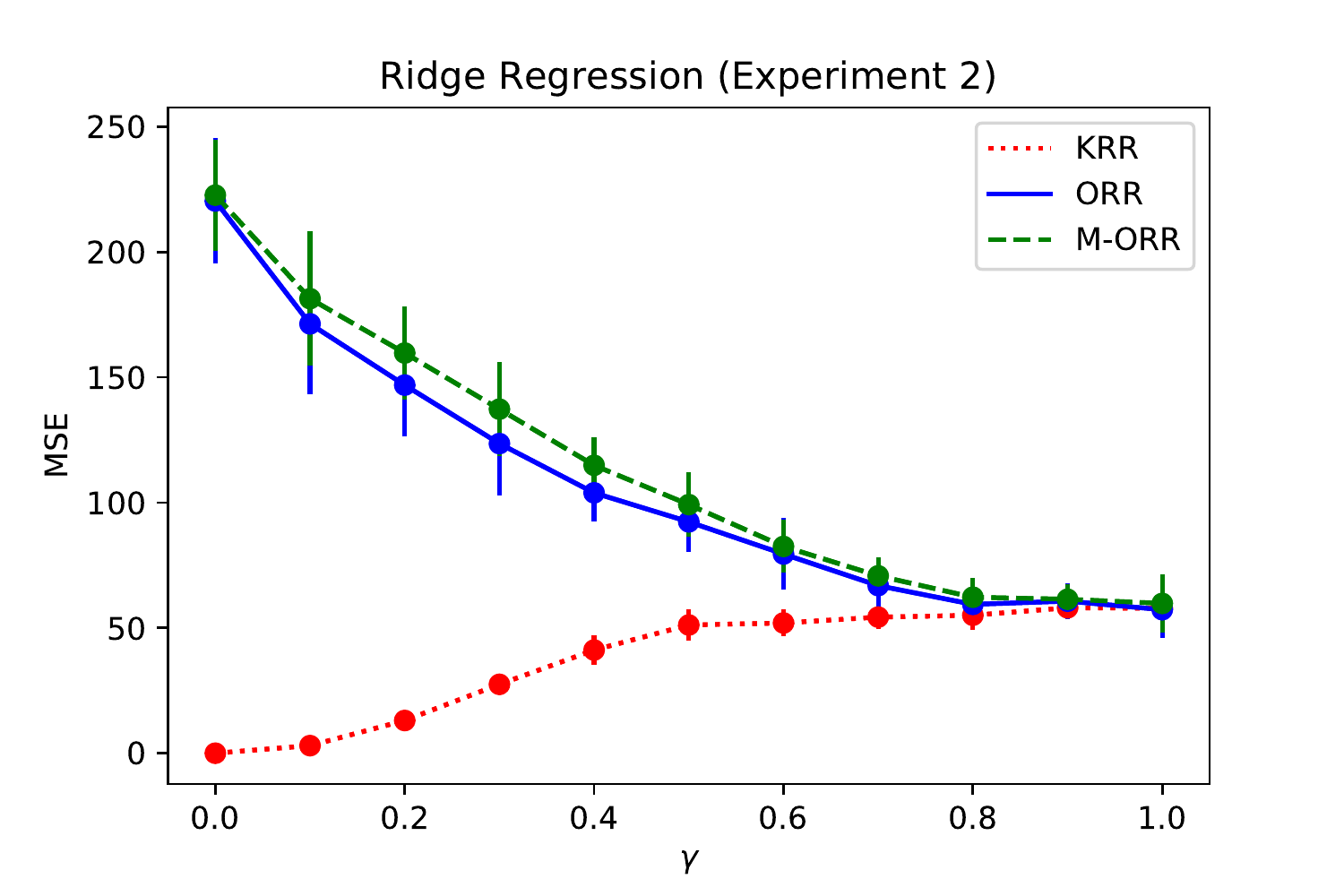}
  \caption{}\label{fig:sub2r}
\end{subfigure}
\caption{Plots \ref{fig:sub1r} and \ref{fig:sub2r} correspond to  Ridge Regression Experiments 1 and 2  respectively. In both plots, the performance of three estimators (KRR,ORR and M-ORR) is plotted against $\gamma$, where $\gamma$ controls the dependence of $X$ on $S$. The case of $\gamma = 0$ corresponds to the highest dependence while $\gamma=1$ corresponds to the case in which $X$ is independent of $S$.}
\label{fig:krr}
\end{figure}

\subsection{Ridge-Regression}
In this section we compare ORR with KRR and the `Markov' version of ORR, M-ORR, which applies the KRR solution to oblivious test features $Z$. We use an RBF kernel with $\sigma=1$. We are  particularly interested in how the dependence of $S$ on $X$ affects the performance and in a comparison of  ORR to M-ORR. We use synthetic data to be able to control the dependence between $S$ and $X$. The basic data generating process is as follows. Sensitive  features $S$ and non-sensitive features $U$ are sampled independently from a  uniform distribution with support $[-5,5]$.  The features $X$ are a convex combination of these two of the form $X = \gamma U + (1-\gamma) S$, $\gamma \in [0,1]$. We consider two ways to generate the response variable $Y$. In \textit{Experiment 1}, 
the response variable is $Y = X^2 + \epsilon$, where $\epsilon$ is normally distributed with variance $0.1$ and is independent of $U$ and $S$. In this case $S\rightarrow X \rightarrow Y$ forms a Markov chain and we expect 
M-ORR to do well. In \textit{Experiment 2}, the variable $S$ influences $Y$ also directly and not only through $X$, i.e. 
$Y = X^2 + S^2 +  \epsilon$. We use here $S^2$ instead of $S$ because $S^2$ is not a zero mean random variable and cannot simply be consumed into the noise term. 

Figure~\ref{fig:krr} shows the results of these experiments. In these experiments, $\gamma$ varies between $0, 0.1, 0.2 \ldots, 1$. For each value of $\gamma$ we generate $500$ data points for ORR and M-ORR to infer the conditional expectations and further $500$ data points are used by all three methods to calculate the ridge regression solution.  For simplicity, we fixed a partition for the conditional expectation: the set $\mathbb{S} = [-5,5]$ is split into a dyadic partition consisting of 16 sets.  
Each method uses a validation set of $100$ data points (which are different from the $500$ training data points) to select the regularization parameter $\lambda$ from $2^{-5}, 2^{-4},\ldots, 2^5$. 
A test set of size $100$ is used to calculate the mean squared error (MSE).
For each $\gamma$ the experiment is repeated $20$ times.  Figure~\ref{fig:krr} reports the average MSE and the standard deviation of the MSE  over these $20$ experiments.

We make the following observations from Figure~\ref{fig:sub1r}.
KRR is the best estimator as it uses the features $X$ directly and not the new features $\mZ$. As $\gamma \rightarrow 1$ both the ORR and M-ORR estimators approach the KRR estimator since the effect of $S$ on $X$ vanishes. Both estimators do not quite reach the performance of the KRR estimator. This is due to the additional uncertainty introduced by estimating the conditional expectations. By definition, the ORR estimator will achieve the best fit of the training data given the new features $\mZ$. We can observe that the M-ORR estimator is performing as well as the ORR estimator even though the M-ORR estimator uses the KRR solution and applies it to $\mZ$. This is due to the fact that $S\rightarrow X\rightarrow Y$ forms a Markov chain. Finally, when $\gamma = 0$ both the M-ORR and ORR estimator achieve an MSE that is very close to the best MSE that can be achieved by a regressor that generates values which are independent of $S$: 
assume that some new features $Z'$ are given which are a function of $X$ and are independent of $S$. 
 when $\gamma=0$   this random variable $Z'$ can  only be independent of $S$ if $Z'$ is a constant. However, if $Z'$ is a constant then the ridge-regressor using $Z'$ is also a constant and the MSE $E(Y - c)^2 = E(S^4) - 2c E(S^2) +c^2 + 0.1  
$ is minimized for $c= E(S^2)$. The minimal value is approximately $56.2$ which is very close to the values of the ORR and M-ORR estimator.      

Figure~\ref{fig:sub2r} shares a few characteristics with Figure~\ref{fig:sub1r} as follows. For $\gamma =0 $ both M-ORR and ORR attain an MSE that is close to the best possible (in the above sense) which is approximately equal to $224.8$. As before, KRR is the overall best estimator and ORR is the best estimator using features $\mZ$. Furthermore, as $\gamma \rightarrow 1$ both estimators become close to the KRR solution. A crucial difference in this experiment is that $S, X, Y$ does not form a Markov chain anymore and the performance of M-ORR
is worse than that of ORR for values of $\gamma$ between $0.2$ and $0.8$. The performance of M-ORR and ORR is essentially the same for $\gamma=0$ and $\gamma =1$. This is not surprising given that when $\gamma = 0$ then $Y = 2 S^2 + \epsilon$ and we are back in the Markov chain setting, while when $\gamma=1$ then $X$ is already independent of $S$.

\section{Discussion}\label{sec:conc}
We have introduced a novel approach to derive oblivious features which approximate non-sensitive features well while maintaining only minimal dependence on sensitive features. 
We make use of Hilbert-space-valued conditional expectations and estimates thereof; our plug-in estimators in this case can be of independent interest in future research, and in turn open grounds for interesting questions involving their guarantees.
The application of our approach to kernel methods  is facilitated by an oblivious kernel matrix which we have derived to be used in place of the original kernel matrix. We characterize the dependencies between the oblivious and the sensitive features in terms of how `close' the sensitive features are to the low-dimensional manifold $\phi[\mathbb{X}]$. One may wonder if this relation can be exploited to further reduce dependencies, and potentially achieve complete independence.
Another important question concerns the interplay between the estimation errors introduced by estimating conditional expectations and the estimation errors introduced by kernel methods which are applied to the oblivious data.

\appendix

\section{Probability in Hilbert spaces: elementary results}
\label{app:prob_in_H}
In this section we summarize a few elementary results concerning random variables that attain values in a separable Hilbert space which we use in the paper. 

\subsection{Measurable functions}\label{subsec:mesfun}
There are three natural definitions of what it means for a function $f:\Omega \rightarrow \cH$
to be measurable. Denote the measure space in the following by $(\Omega,\mathcal{A})$ with the understanding that these definitions apply, in particular, to $\Omega = \mathbb{R}^d$ and $\mathcal{A}$ being the corresponding Borel $\sigma$-algebra.
\begin{enumerate} 
\item \label{prop:bmf} $f$ is \textit{Bochner-measurable} iff $f$ is the point-wise limit of a sequence of simple functions, where $\mS:\Omega \rightarrow \cH$ is a simple function if it can be written as
\[
\mS(\omega) = \sum_{i=1}^n h_i \times \chi A_i(\omega) 
\]
for some $n\in\mathbb{N}$, $A_1,\ldots, A_n \in \mathcal{A}$ and $h_1,\ldots, h_n \in \cH$.

\item  $f$ is \textit{strongly-measurable} iff $f^{-1}[B] \in \mathcal{A}$ for every Borel-measurable subset $B$ of $\cH$. The topology that is used here is the norm-topology.

\item \label{prop:weakm} $f$ is \textit{weakly-measurable} iff for every element $h\in \cH$ the function $\langle h, f \rangle :\Omega \rightarrow \mathbb{R}$ is measurable in the usual sense (using the Borel-algebra on $\mathbb{R}$).
\end{enumerate}

All three definitions of measurability are equivalent in our setting. We call a function $f:\Omega \rightarrow \cH$ a \textit{random variable} if it is measurable in this sense.

The main example in our paper is $f = \phi(X)$. This is a well defined random variable whenever $X:\Omega \rightarrow \mathbb{R}^d$ and  $\phi:\mathbb{R}^d \rightarrow \cH$ are both Borel-measurable. 

\section{Hilbert space-valued conditional expectations}
\label{sec:HS_space_valued}
\subsection{Basic properties}
We recall a few important properties of Hilbert space valued conditional expectations. These often follow from properties of real-valued conditional expectations through `scalarization' \citep{PIS16}. In the following, let $\mX, \mZ \in \mathcal{\altL}^2(\Omega,\mathcal{A},P;\cH)$ and $\mathcal{B}$ some $\sigma$-subalgebra of $\mathcal{A}$.  Due to \citet{PIS16}[Eq. (1.7)], for any $f\in \cH$  
\begin{equation} \label{eq:scalarization}
\langle f, E^\mathcal{B} \mX \rangle = E^\mathcal{B} \langle f, \mX \rangle \text{\quad \quad (a.s.)}
\end{equation}
and the right hand side is just the usual real-valued conditional expectation. It is also worth highlighting that the same holds for the Bochner-integral $E(\mX)$, i.e. for any $f\in \cH$, $\langle f, E(\mX) \rangle = E \langle f,\mX \rangle$. This can be used to derive properties of $E^\mathcal{B} \mX$. For instance, since $E( E^\mathcal{B} \langle f,\mX \rangle ) = E \langle f,\mX \rangle$ is a property of real-valued conditional expectations we find right away that
\[
\langle f, E(\mX) \rangle = E \langle f,\mX \rangle = E(E^\mathcal{B} \langle f,\mX \rangle) = E \langle f, E^\mathcal{B} \mX \rangle = \langle f, E(E^\mathcal{B} \mX) \rangle. 
\]
Because $E(\mX)$ and $E(E^\mathcal{B} \mX)$ are elements of $\cH$ and for all $f\in \cH$ 
\[
\langle f, E(\mX) - E(E^\mathcal{B} \mX) \rangle =0
\]
it follows that $E(\mX) = E(E^\mathcal{B} \mX)$. 

Another result we need is that if $\mZ$ is $\mathcal{B}$-measurable then 
\[
E^\mathcal{B} \langle \mX, \mZ \rangle = \langle E^\mathcal{B} \mX, \mZ  \rangle \text{\quad \quad (a.s.)}.
\]
Showing this needs a bit more work. Since $\mZ \in \mathcal{L}^2(\Omega,\mathcal{B},P;\cH)$ there exist  $\mathcal{B}$-measurable simple functions $U_n$ such that $U_n$ converges point-wise to $\mZ$, $\lim_{n\rightarrow \infty} \|U_n^\bullet - \mZ^\bullet\|_2 = 0$ and the sequence fulfills $\|U_n\| \leq 3 \|\mZ\|$ for all $n\in\mathbb{N}$ \citep{PIS16}[Prop.1.2]. Consider some $n$ and write \[U_n = \sum_{i=1}^m h_i \times \chi A_i,\] for a suitable $m\in \mathbb{N}, h_i \in \cH, A_i \in \mathcal{B}$, then 
\begin{align*}
E^\mathcal{B} \langle \mX, U_n \rangle &= \sum_{i=1}^m E^\mathcal{B} (\langle \mX, h_i \rangle \times \chi A_i) \\ 
&=\sum_{i=1}^m (E^\mathcal{B} \langle \mX, h_i \rangle) \times \chi A_i \\ &= \sum_{i=1}^m  \langle  E^\mathcal{B} \mX, h_i \rangle \times \chi A_i  \\ &= \langle  E^\mathcal{B} \mX, U_n \rangle \text{\quad \quad (a.s.)},
\end{align*}
because $\chi A_i$ is $\mathcal{B}$-measurable. For the right hand side point-wise convergence of $U_n$ to $\mZ$ tells us that 
for all $\omega\in \Omega$  we have $\lim_{n\rightarrow \infty} \|U_n(\omega) - \mZ(\omega)\| =0$. Because $E^\mathcal{B} \mX^\bullet \in L^2(\Omega,\mathcal{A},P;\cH)$ we also know that $E^\mathcal{B} \mX$ is finite almost surely. Therefore, for $\omega$ in the corresponding co-negligible set,
\[
\lim_{n\rightarrow \infty} |\langle  (E^\mathcal{B} \mX)(\omega), U_n(\omega) \rangle - \langle  (E^\mathcal{B} \mX)(\omega), \mZ(\omega) \rangle| \leq \lim_{n\rightarrow \infty} \|(E^\mathcal{B} \mX)(\omega)\| \| U_n(\omega) - \mZ(\omega)\| =0
\]
and  $\lim_{n\rightarrow \infty} \langle  E^\mathcal{B} \mX, U_n \rangle = \langle  E^\mathcal{B} \mX, \mZ \rangle$ almost surely. 
 
By the same argument it follows that 
$\lim_{n\rightarrow \infty} \langle  \mX, U_n \rangle = \langle  \mX, \mZ \rangle$ almost surely. Let $h_n = \langle \mX, U_n \rangle$ and  $h = \langle  \mX, \mZ \rangle$ 
then $|h_n - h| \leq \|\mX\| \|U_n - \mZ\| \leq 3 \|\mX\| \|\mZ\|$. Furthermore, $|h_n| \leq |h| + 3 \|\mX\| \|\mZ\| \leq 4 \|\mX\| \|\mZ\| \leq 4 (\|\mX\|^2 + \|\mZ\|^2)$. The right hand side lies in $\mathcal{\altL}^1$ and dominates $h_n$.
Using \citet{SHI89}[II.\S 7.Thm.2(a)], we conclude that 
\[
\lim_{n\rightarrow \infty} E^\mathcal{B} \langle \mX, U_n \rangle = E^\mathcal{B} \langle \mX, \mZ \rangle \text{\quad (a.s.)}
\]
and the result follows. 
 
The operator $E^\mathcal{B}$ is also idempotent and self-adjoint, i.e.
\[
E^\mathcal{B} \mX = E^\mathcal{B}( E^\mathcal{B} \mX) \text{\enspace (a.s) \quad and \quad} \bl \mX^\bullet, E^\mathcal{B} \mZ^\bullet  \br
= \bl E^\mathcal{B} \mX^\bullet,  \mZ^\bullet  \br.
\]

\subsection{Representation of conditional expectations}
A well known result in probability theory states that a conditional expectation $E^S X$ of a real-valued random variable $X$ given another real-valued random variable $S$ can be written as $g(S)$ with some suitable measurable function $g:\mathbb{R}\rightarrow \mathbb{R}$. This result generalizes to our setting. Here, we include the generalized result together with a short proof for reference.

\begin{lem} \label{prelem:cond_exp_g}
Consider a probability space $(\Omega,\mathcal{A},P)$, and let $\cH$ be a separable Hilbert space. Let $S: \Omega \rightarrow \mathbb{R}^d$ be a random variable and suppose that $\eta: \Omega \rightarrow \cH$ is a $\sigma(S)$-measurable function. There exists a Bochner-measurable function $g:\mathbb{R}^{d} \rightarrow \cH$ such that 
\[
\eta = g \circ S \text{\quad almost surely.}
\]
\end{lem}
\begin{proof}
We first show the statement for simple functions, and observing that any arbitrary Bochner-measurable function can be written as the point-wise limit of a sequence of simple functions, we extend the result to arbitrary $\eta$. 

First, assume that $\eta:= h \chi A$ for some $h \in \cH$ and $A \in \sigma(S)$. 
Since $S$ is measurable with respect to $\B({\mathbb R}^d)$ there exists some $B \in \B({\mathbb R}^d)$ such that $\{\omega: S(\omega) \in B\} = A$.  Define $g: {\mathbb R}^d \rightarrow \cH$ as $g:=h \tilde \chi B$, where $\tilde \chi$ denotes the indicator function on $\mathbb{R}^d$.
We obtain, $\eta(\omega)=h \chi A(\omega)= h \tilde \chi B(S(\omega))$ so that $\eta= g \circ S$. 
Next, let
$\eta:=\sum_{i=1}^m h_i \chi A_i$ for some $m \in \mathbb N$,~ $h_1, \dots, h_m \in \cH$ and $A_1, \dots, A_m \in \sigma(S)$. As above, by measurability of $S$, there exists a sequence $B_1, \dots, B_m \in \B(\mathbb R^d)$ such that $A_i=S^{-1}[B_i],~i \in 1,\dots, m$. 
It follows that $\eta(\omega)=\sum_{i=1}^m h_i \chi A_i(\omega)=\sum_{i=1}^m h_i \tilde \chi B_i(S(\omega)),~ \omega \in \Omega$; hence, $\eta = g \circ S$ for 
$g = \sum_{i=1}^m h_i \tilde \chi B_i$.  
Observe that in both cases $g$ is trivially Bochner-measurable by construction, since it is a simple function.

Let $\eta: \Omega \rightarrow \cH$ be an arbitrary Bochner-measurable function that is also measurable with respect to $\sigma(S)$. There exists a sequence of simple functions $\eta_n,~n \in \mathbb N$ such that for every $\omega \in \Omega$ we have $$\eta (\omega)=\lim_{n\rightarrow \infty} \eta_n (\omega).$$ Since each $\eta_n$ is a simple function, by our argument above, there exists a sequence of Bochner-measurable functions $g_n: \mathbb R^d \rightarrow \cH$ such that 
$\eta_n= g_n \circ S$ where for each $n \in \mathbb{N}$ the function 
$g_n$ is simple of the form $g_n=\sum_{i=1}^{m_n} h_{i,n} \tilde \chi B_{i,n}$ for some $m_n \in \mathbb N$ and a sequence of functions $h_{1,n},\dots,h_{m_n,n} \in \cH$ and a sequence of Borel sets $B_{1,n}, \dots, B_{m_n,n} \in \mathcal B(\mathbb R^d)$. 

Denote by $B:=\{S(\omega): \omega \in \Omega\} \subset \mathbb R^d$ the image of $S$, and observe that for each $x \in B$ $\lim_{n\rightarrow \infty}g_n(x)$ exists. To see this, note that by construction, for each $x \in B$ we have $x=S(\omega)$ for some $\omega \in \Omega$, thus, it holds that
\begin{align*}
    \lim_{n \rightarrow \infty}g_n(x)&=\lim_{n \rightarrow \infty}g_n(S(\omega))
    =\lim_{n\rightarrow \infty}\eta_n(\omega)
    =\eta(\omega).
\end{align*}
Moreover, we have  $P(S^{-1}[B])=P(\{\omega \in \Omega: S(\omega) \in B\}) = P(\Omega)=1$.
Define 
$g: {\mathbb R}^d \rightarrow \cH $ as
\begin{equation}
    g(x):= \begin{cases}
    \lim_{n\rightarrow \infty} g_n(x) &x \in B\\
    0 & x \notin B
    \end{cases}
\end{equation}
Thus, for each $\omega \in \Omega$ with probability $1$, we have 
\begin{equation}
\eta(\omega)=\lim_{n\rightarrow \infty} \eta_n(\omega)=\lim_{n \rightarrow \infty} g_n(S(\omega)) = g(S(\omega)),
\end{equation}
so that $\eta=g \circ S$ almost surely.
On the other hand, since by definition, $g$ is the pointwise limit of a sequence of simple functions $g_n$, it is Bochner-measurable, (see Property~\ref{prop:bmf} in Section~\ref{subsec:mesfun}) and the result follows. 
\end{proof}

\begin{lem} \label{lem:cond_exp_g}
Consider a separable Hilbert space $\cH$, a probability space $(\Omega,\mathcal{A},P)$ , a Bochner-integrable random variable $\mX:\Omega \rightarrow \cH$  and  a random variable $S:\Omega \rightarrow \mathbb{R}^{d}$. There exists a Bochner-measurable function $g:\mathbb{R}^{d} \rightarrow \cH$ such that 
\[
E^S\mX = g\circ S \text{\quad almost surely.}
\]
\end{lem}
\begin{proof}
Observing that by definition of conditional expectation, $E^S\mX$ is a $\sigma(S)$-measurable function from $\Omega$ to $\cH$, the result readily follows from Lemma~\ref{prelem:cond_exp_g}.
\end{proof}

\section{Proofs}

\subsection{Proof of Proposition~\ref{prop:multival}} \label{app:multival}
\begin{proof}
Let $\cM:=\phi[\mathbb{X}]$ denote the manifold corresponding to the image of $\mathbb{X}$ under $\phi$, equipped with the subspace topology and corresponding Borel $\sigma$-algebra $\B(\cM)$. Define the {\em metric projection map} $\pi:\cH \rightrightarrows \cM $ as a multi-valued function such that
\begin{align}\label{eq:projpi}
\pi(g) = \left \{h \in \cM: \bn h-g \bn = \min_{h' \in \cM} \bn h'-g \bn \right \}.
\end{align}
Note that the $\min$ operator in Equation \eqref{eq:projpi} is well-defined since by definition $h=\phi(x)$ for some $x \in \mathbb X$, the space $\mathbb X$ is compact and $\phi$ is a continuous function. 
Observe that $\pi$ is not a function, but a multi-valued function which assigns to each element $g \in \cM$ a subset of $\cM$, see, e.g. \cite[Section 6.1]{BEE93} for more on this notion. 

\paragraph{$\pi$ maps to non-empty compact subsets of $\cM$.} For each $g \in \cH$, set $f_g(h) = \bn h-g\bn$ with $h\in \cM$, and note that it is a continuous function  from $\cM$ to $\mathbb R$.  Let $m(g):=\min_{ h \in \cM} \bn h-g \bn$, which, by the above argument, is well-defined, and observe that,  since $\{m(g)\}$ is a closed subset of $\mathbb{R}$, then $\pi(g)=f_g^{-1}[\{m(g)\}]$ is a closed subset of $\cM$. Since $\cM$ is compact as the continuous image of the compact space $\mathbb{X}$ it follows that $\pi(g)$ is compact. 
\paragraph{$\pi$ is upper-semicontinuous.}
As follows from the standard definition, see, e.g. \citet[Definition 6.2.4 and Theorem 6.2.5]{BEE93}, the multi-valued function $\pi$ is said to be upper-semicontinuous at a point $g_0 \in \cH$ if for any open subset $V$ of $\cM$ such that $\pi(g_0) \subseteq V$ it holds that $\pi(g)\subseteq V$ for each $g$ in some neighbourhood of $g_0$.\footnote{Upper-semicontinuity is also referred to as upper-hemicontinuity for multi-valued functions in the literature.} To show the upper-semicontinuity of $\pi$ we proceed as follows. 
Take $g_0 \in \cH$. 
Let $V$ be an open subset of $\cM$ such that $\pi(g_0)\subseteq V$. 
Denote by $\widetilde{\cM}:=\cM \setminus V$. 
Note that $\widetilde{\cM}$ is compact since it is a closed subset of $\cM$ which is in turn compact. Therefore, in much the same way as for $\cM$, the $\min$ operator is well-defined for $\widetilde{\cM}$, i.e.  the minimum $\widetilde{m}(g_0):=\min_{h \in \widetilde{\cM}} \bn g_0 - h\bn$ exists. Moreover, since $\pi(g_0) \subseteq V$ and $V \cap \widetilde{\cM} =\emptyset$, it holds that $\widetilde{m}(g_0)>m(g_0)$. Therefore, there exists some $\delta>0$ such that $\widetilde{m}(g_0) \geq \delta + m(g_0)$. 
Consider an open ball $B_{g_0}(\delta/3)$ of radius $\delta/3$ around $g_0$. For every $g \in B_{g_0}(\delta/3)$ and all $h \in \pi(g_0)$ we have 
\begin{align}
\bn g-h \bn 
&\leq \bn g_0-g \bn + \bn g_0 -h \bn \leq \delta/3+m(g_0).\end{align}
On the other hand, we have
\begin{align}
\min_{h' \in \widetilde{\cM}}\bn g-h' \bn  \geq \big | \bn g-g_0 \bn - \bn g_0 - h' \bn \big | \geq m(g_0)+2\delta/3.
\end{align}
This implies that $\pi(g)\cap\widetilde{\cM}=\emptyset$ because there are already better candidates (closer to $g$) in $\pi(g_0)$ which is in turn contained in $V$ and thus does not intersect $\widetilde{\cM}$). 
Hence, it must hold that $\pi(g) \subseteq V$. Finally, since the choice of $g \in B_{g_0}(\delta/3)$ is arbitrary, it follows that for all $g \in B_{g_0}(\delta/3)$ we have
$\pi(g) \subseteq V$ and $\pi$ is upper-semicontinuous.

\paragraph{$\phi$ is a homeomorphism. } To see this, note that $\phi$ is bijective and continuous since the kernel is positive definite and continuous: it is by definition surjective and it is injective since $\phi(x) = \phi(y)$ for $x \not =y$ would imply that $a_1^2 \|\phi(x)\|^2 - 2 a_1 a_2 \langle \phi(x),\phi(y)\rangle + a_2^2 \|\phi(y)\|^2 = 0$ when $a_1 = a_2 = 1$.
The statement follows now from \citet[Theorem 3.1.13]{ENGEL} since $\mathbb X$ is compact and $\mathcal{M}$ is a Hausdorff space. 

\paragraph{Measurable selection.}  Since $\pi$ is upper-semicontinuous and maps to compact sets it is \textit{usco-compact} \cite[Definition 422A]{FREM}. This implies that $\pi$ is measurable as a function from $\cH$ to the compact subsets of $\cM$ where the latter is equipped with the Vietoris topology and the corresponding Borel algebra \cite[Proposition 5A4Db]{FREM}. Furthermore, there exists a Borel-measurable function $f$ from the compact, non-empty, subsets of $\cM$ to  $\cM$ such that $f(K) \in K$ for every compact, non-empty, subset $K$ of $\cM$. Define $W' = f(\pi(Z +h^*))$ then $W = \phi^{-1}(W')$ is the continuous image of the measurable function $W'$ and $W$ has the stated properties.
\end{proof}

\subsection{Proof of Proposition~\ref{prop:alpha_dep}}
 \label{app:proof_dep}
\begin{proof}
\textbf{(a)} Let $W$ be the random variable provided by Proposition \ref{prop:multival} and let $\bm{W} = \phi(W) - h^*$. Then $\bn \mZ - \bm{W} \bn_2  = d(\mZ + h^*,\mathcal{M})$.
 Observe that two applications of the Cauchy-Schwarz inequality yield
\begin{align*}
E(|\langle f, \mZ \rangle  - f(W) - \langle f, h^* \rangle|  \times \chi B) &\leq E| \langle f, (\mZ - \phi(W) - h^*) \times \chi B \rangle | \\&\leq \|f\| E(\chi B \times \|\mZ - \mb W \|) \\ &\leq \sqrt{P(B)} \|f\| \bn \mZ - \mb W\|_2 
\end{align*}
for all $f \in \cH$. Similarly, for any $f\in \cH$ it holds that 
\begin{equation*}
E| \langle f,\mZ -\phi(W) - h^*\rangle| \leq \|f\| E \|\mZ - \mb W \| \leq \|f\| \bn \mZ - \mb W \bn_2.
\end{equation*} 
Noting that $\mZ$ is $\cH$-independent of $S$ we find that  for any $f\in \cH$ and $B\in \sigma(S)$
\begin{align*}
&|E (f(W) \times \chi B) - E f(W) P(B)|  \\
&= |E((f(W) - \langle f,h^*\rangle) \times \chi B) - E (f(W) -\langle f,h^* \rangle) P(B)| \\
&\leq | E(\langle f,\mZ \rangle \times \chi B) - E \langle f,\mZ \rangle  P(B)| + (1+\sqrt{P(B)}) \|f\| \bn \mZ - \mb W \bn_2\\ 
&\leq 2 \|f\| \, d(\mZ + h^*,\mathcal{M}) .
\end{align*}
\textbf{(b)} For  $C\in \sigma(W)$ let $D$ be the image of $C$ under $W$, i.e. $D=W[C]$, $D \subset \mathbb{X}$. For $f \in \cH$ let
\[
\xi_C(f) = \| \chi D(W) - f(W) \|_2.
\]
Now, for any $B \in \sigma(S)$,
\[
|P(C \cap B) - E (f(W) \times \chi B)| \leq P(B)^{1/2} (E(\chi D(W) - f(W))^2)^{1/2} \leq  \xi_C(f).  
\]
Moreover, we have $|P(C) - E f(W)| \leq \xi_C(f)$. Hence, for any $f\in \cH$ it holds that 
\begin{align*}
|P(C\cap B) - P(C)P(B)| &\leq 2\xi_C(f)  + |E(f(W) \times \chi B) - E f(W) P(B)| \\
&\leq 2 (\xi_C(f) + \|f\| d(\mZ + h^*,\mathcal{M}) ).
\end{align*}
This proves the first part of the proposition.

\textbf{(c)} For the second part: by assumption for $A\in \sigma(\mZ)$ there exists a $C \in \sigma(W)$ such that $P(A \triangle C) \leq c$. For any such $C$ we have that 
$|P(C) - P(A)| \leq P(C \triangle A) \leq c$ and 
\[
|P(C\cap B) - P(A\cap B)| \leq P( (C\triangle A) \cap B) \leq c.
\]
Hence,
$|P(A\cap B) - P(A)P(B)| \leq 2 c +  2(\xi_C(f) + \|f\|  d(\mZ + h^*,\mathcal{M}))$ for all $f \in \cH$. Taking the infimum over $f$ and $C$ proves the second part of the proposition.
\end{proof}
\subsection{Proof of Proposition \ref{prop:dist_chaining}}\label{app:dist_chaining}
First note that since $\phi$ is continuous and $\mathbb X$ is compact it follows that $\rho$ is finite and 
\begin{align}
\|\mZ\|
&=\|\phi(X)-E^S \phi(X)+E \phi(X)\| \notag \\
&\leq \|\phi(X)\| + \| E^S \phi(X)\|+\| E \phi(X)\| \notag \\
&\leq  \|\phi(X)\| +  E^S \|\phi(X)\|+ E \|\phi(X)\| \label{eq:EphiBoch}\\
& \leq 3\rho \notag
\end{align}
where \eqref{eq:EphiBoch} follows from \citet[Theorem II.4]{DIES77} and \citet[Proposition 1.12]{PIS16}. 
Let ${\mZ}_i,~i \leq n$ be $n$ independent copies of $\mZ$ and define 
$Y_i:=\min_{h \in \cM}\|{\mZ}_i+h^*-h\|, ~i \leq n$, and 
$Y:=\min_{h \in \cM}\|\mZ+h^*-h\|$.  
By Hoeffding's inequality we have, 
\begin{align*}
\Pr \left (\left |\frac{1}{n}\sum_{i=1}^n Y_i-E Y \right | \geq \epsilon \right ) 
&\leq 2\exp\Bigl(-\frac{2n \epsilon^2}{25 \rho^2}\Bigr)
\end{align*} 
and the result follows.

\subsection{Proof of Proposition \ref{lem:best_approx}} \label{app:best_approx}
\begin{proof}
\textbf{(a)} We first show that 
\begin{equation} \label{eq:supp_inner_zero}
\bl E^S \mX^\bullet, (\mZ')^\bullet  \br = \bl E(\mX^\bullet), (\mZ')^\bullet \br.
\end{equation}
$E^S \mX^\bullet$ is an element of $L^2(\Omega,\sigma(S),P;\cH)$ and there exists a sequence of simple function $\{U_n\}_{n\in \mathbb{N}}$ such that $\lim_{n\rightarrow \infty} \bn U_n^\bullet - \mX^\bullet\bn = 0$. In particular, $\lim_{n\rightarrow \infty} \langle U_n^\bullet, (\mZ')^\bullet \rangle = \bl \mX^\bullet, (\mZ')^\bullet  \br$ and $\| E(U_n^\bullet) - E(\mX^\bullet) \| \leq E \|U_n^\bullet -\mX^\bullet\| = \bn U_n^\bullet -\mX^\bullet \bn$ goes to zero in $n$.  
Consider some $U_n = \sum_{i=1}^{m} h_{i} \times \chi A_{i}$, $h_{i} \in \cH, A_{i} \in \sigma(S)$, $m\in \mathbb{N}$, and observe that
\[
\bl U_n^\bullet, (\mZ')^\bullet \br = \sum_{i=1}^{m} E\langle h_{i} \times \chi A_i, \mZ' \rangle = \sum_{i=1}^{m} E( \langle h_{i}, \mZ'  \rangle \times \chi A_i) = \sum_{i=1}^{m} E\langle h_{i},\mZ' \rangle \times E(\chi A_{i}), 
\]
using the assumption on $\mZ'$. The assumption can be applied because $\chi A_i$ is $\sigma(S)$-measurable, and, hence, can be written as a function of $S$ \cite{SHI89}[II.\S 4.Thm.3].
Now, 
\[
\sum_{i=1}^{m} E\langle h_{i},\mZ' \rangle \times E(\chi A_{i}) = E \langle \sum_{i=1}^{m} h_{i} \times E(\chi A_{i}),\mZ' \rangle = E \langle E(U_n^\bullet),\mZ' \rangle   
\]
and $\bl U_n^\bullet, (\mZ')^\bullet \br = \bl E(U_n^\bullet), (\mZ')^\bullet \br$. Equation \eqref{eq:supp_inner_zero}  follows since $U_n^\bullet$ converges to $\mX^\bullet$ and $E(U_n^\bullet)$ converges to $E(\mX^\bullet)=0$ in $L^2(\Omega,\mathcal{A},P;\cH)$.

\textbf{(b)}
Since $\bl E^S \mX^\bullet, (\mZ')^\bullet \br = \bl E(\mX^\bullet), (\mZ')^\bullet \br$ and $\bl \mX^\bullet,E^S \mX^\bullet \br = \bn E^S \mX^\bullet\bn^2$ it follows right away that 
\begin{align*}
&\bn \mX^\bullet - (\mZ')^\bullet \bn^2 \\
&= \bn \mX^\bullet - \mZ^\bullet \bn^2 + 2 \bl E^S \mX^\bullet - E(\mX^\bullet), \mX^\bullet - E^S\mX^\bullet + E(\mX^\bullet) - (\mZ')^\bullet \br \\ &\qquad \qquad \qquad \qquad\qquad\qquad\qquad\qquad\qquad\qquad\qquad\qquad\qquad\qquad\qquad\qquad+ \bn \mZ^\bullet - (\mZ')^\bullet \bn^2   \\
&= \bn \mX^\bullet - \mZ^\bullet \bn^2 + \bn \mZ^\bullet - (\mZ')^\bullet \bn^2.
\end{align*}
Hence, $\mZ$ is a minimizer and it is almost surely unique because $\bn \mZ^\bullet - (\mZ')^\bullet \bn^2$ is only zero if $\mZ^\bullet = (\mZ')^\bullet$.
\end{proof}

\subsection{Proof of Proposition \ref{Prop:Est_error_cond}} \label{app:proof_est_error}
\begin{proof}
\textbf{(a)}
In the following, let $s_1,\ldots, s_l$ be the values $S$ can attain. Furthermore,  let $f_i = E_n(\phi(X) | S= s_i)  - E(\phi(X) | S = s_i)$, and let $\cF = \sigma(X_1,S_1,\ldots, X_n, S_n)$. Each $f_i$ is $\cF$-measurable. Observe that for $i\not= j$, 
\begin{align*}
&E^\cF( \langle f_i \times \chi \{S = s_i\}, f_j \times \chi \{S = s_j\} \rangle) = E^\cF( \langle f_i, f_j \rangle \times \chi \{S = s_i, S =  s_j\} ) \\
&= \langle f_i, f_j \rangle \cdot E^\cF(\chi \{S = s_i, S = s_j \} ) = \langle f_i, f_j \rangle \cdot P( S = s_i, S = s_j ) 
=0
\end{align*}
since $f_i,f_j$ are $\cF$-measurable and $S$ is independent of $\cF$.
Hence,
\begin{align*}
E^\cF( \| E_n^S \phi(X) - E^S \phi(X)\|^2) 
&= E^\cF \Bigl(\Bigl\|\sum_{i=1}^l f_i \times \chi \{S = s_i\}   \Bigr\|^2 \Bigr)     \\
 &= \sum_{i=1}^l E^\cF( \|f_i \times \chi \{S = s_i\}\|^2) \\
 &= \sum_{i=1}^l E^\cF( \|f_i \|^2 \times \chi \{S = s_i\}) \\
&= \sum_{i=1}^l \|f_i \|^2  P(S = s_i) \\
&= \sum_{i=1}^l P(S = s_i) \sup_{\|h\|\leq 1} |E_n(h(X) | S = s_i)  - E(h(X) |  S =  s_i)|^2.  
\end{align*}
\textbf{(b)} For each $i$ either $P(S=s_i) = 0$ or 
\[
 \sup_{\|h\|\leq 1} |E_n(h(X) | S = s_i)  - E(h(X) |  S =  s_i)|^2 \in O_P^*(n^{-1})
\]
using \citet{SG}. Since there are only $l$-many terms in the sum this result carries over to the whole sum.
\end{proof}
\subsection{Proof of Proposition \ref{Prop:Est_error_cond_cont}} \label{app:proof_est_error_cont}
\begin{proof}
Recall the notation $\mathfrak{D}_{\ell}:=\{\Delta_i: i \in 1,\dots,{\ell^d}\},~\ell \in \mathbb N$ where $\Delta_1,\Delta_2,\dots,\Delta_{{\ell^d}}$ are the dyadic cubes $\Delta_1,\Delta_2,\dots,\Delta_{{\ell^d}}$ of side-length $1/{\ell}$ discretizing  $\mathbb S$.
Let $\mathcal G:=\sigma(\{S^{-1}[\Delta] : ~\Delta \in \mathfrak{D}_{\ell}\})$ and choose a Bochner measurable $g:\mathbb{S} \rightarrow \cH$ according to 
Lemma \ref{lem:cond_exp_g} such that $g(S) = E^S \phi(X) $ (a.s.).
Since $\mathcal G \subseteq \sigma(S)$ we have,
\begin{equation}\label{eq:GSordering}
E^{\mathcal G}\phi(X)=E^{\mathcal G}(E^{S}\phi(X)))=E^{\mathcal G}(g(S))~\textit{\quad almost surely}.
\end{equation}
In the following, we use $g \circ S$ instead of $g(S)$ for readability. With probability one it holds that, 
\begin{align*}
E^{\mathcal F} \|g \circ S &- E^\mathcal{G} (g \circ S)\|^2 \\
&=E^{\mathcal F}\Bigl(   \sum_{\Delta \in \mathfrak{D}_{\ell}}\|g \circ S - E^\mathcal{G} (g \circ S)  \|^2 \chi\{S \in \Delta\} \Bigr) \\
&= \sum_{\Delta \in \mathfrak{D}_{\ell}}E^{\mathcal F} \|(g \circ S- E^\mathcal{G} (g \circ S))\chi\{S \in \Delta\}\|^2 \\
&=\sum_{\Delta \in \mathfrak{D}_{\ell}} E^{\mathcal F} \|(g \circ S-\sum_{\Delta' \in \mathfrak{D}_{\ell}}E(g \circ S|S \in \Delta')\chi\{S \in \Delta'\})\chi\{S \in \Delta\}\|^2\\
&=\sum_{\Delta \in \mathfrak{D}_{\ell}} E^{\mathcal F} \left ( \|g \circ S -E(g \circ S|S \in \Delta) \|^2 \chi\{S \in \Delta\}\right)
\end{align*}
By \citet[II.Corollary 8]{DIES77} for any $\Delta \in \mathfrak{D}_{\ell}$  it holds that the conditional expectation of $g$ given $\Delta$ is
in the closed convex hull of $g[\Delta]:=\{g(s):s \in \Delta\}$. That is,
\begin{equation*}
\frac{1}{\mu(\Delta)} \int_{\Delta} g \,d\mu \in \cch(g[\Delta])
\end{equation*}
This means that for every $\epsilon >0$ there exist $k \in \mathbb N$, and some $s_1,\dots,s_k  \in \Delta$ and $\alpha_1,\dots, \alpha_k  >0$ with $\sum_{j=1}^k \alpha_i=1$ such that 
\begin{equation*}
\Bigl\|\frac{1}{\mu(\Delta)} \int_{\Delta} g \, d\mu -\sum_{j=1}^k \alpha_j g(s_j) \Bigr\|^2 \leq \epsilon.
\end{equation*}
Let $D:=S^{-1} [\Delta]$.
We obtain
\begin{equation*}
\frac{1}{\mu(\Delta)} \int_{\Delta} g \, d\mu = \frac{1}{P(D)} \int_{D} (g \circ S) \, dP = E(g \circ S|S \in \Delta).
\end{equation*}
Since $g \circ S$ is assumed to be $L$-Lipschitz-continuous, for all $\Delta \in \mathfrak{D}_{\ell}$ we have
\begin{align*}
  \|g \circ S&-\sum_{j=1}^k \alpha_j g(s_j)\|^2\chi\{S \in \Delta\} \\
 & \leq \sup_{s \in \Delta}\|\sum_{j=1}^k \alpha_j (g(s)-g(s_j))\|^2\\
 & \leq  \sup_{s \in \Delta}(\sum_{j=1}^k\alpha_j \|g(s)-g(s_j)\|)^2 \\
 & \leq  \Bigl(\sum_{j=1}^k\alpha_j \sup_{s \in\Delta} \|g(s)-g(s_j)\|\Bigr)^2\\
 &\leq L^2\Bigl(\sum_{j=1}^k \alpha_j \sup_{s \in \Delta}\|s-s_j\|\Bigr)^2\\
 &\leq d L^2 {\ell}^{-2}.
\end{align*}
Moreover, noting that $\chi\{\cdot\} = \chi^2\{\cdot\}$ we obtain,  
$
\|g \circ S-\sum_{j=1}^k \alpha_j g(s_j)\|\chi\{S \in \Delta\} \leq  L \sqrt{d}{\ell}^{-1}.
$
It follows that,
\begin{align*}
&\|(g \circ S-E(g \circ S  |S \in \Delta))\chi\{S \in \Delta\}\|^2 \nonumber \\
&=\|(g \circ S-E(g \circ S|S \in \Delta))\|^2\chi\{S \in \Delta\}  \nonumber \\
&\leq \Bigl( \bigl\|g \circ S-\sum_{j=1}^k \alpha_j g(s_j)\bigr\| + \bigl\|\sum_{j=1}^k \alpha_j g(s_j) - E(g\circ S| S \in \Delta)\bigr\| \Bigr)^2 \chi\{S \in \Delta\}  \\
& \leq  \Bigl( \bigl\|(g \circ S-\sum_{j=1}^k \alpha_j g(s_j)) \bigr\| + \epsilon  \Bigr)^2 \chi\{S \in \Delta\}.  \\
\end{align*}
Since this holds for every $\epsilon>0$ we have,
\begin{equation*}
\|(g \circ S-E(g \circ S  |S \in \Delta))\chi\{S \in \Delta\}\|^2  \leq d L^2 {\ell}^{-2}.
\end{equation*}
Observe that for $\Delta \not = \Delta'$, $\Delta, \Delta' \in \mathfrak{D}_\ell$,
\begin{align*}
E^{\mathcal F}\Bigl( \|g \circ S-E(g \circ S|S \in \Delta)\| \chi\{S \in \Delta\} \times
\|g \circ S-E(g \circ S|S \in \Delta')\| \chi\{S \in \Delta'\}
 \Bigr) = 0
\end{align*}
and
\begin{align*}
E^{\mathcal F}\|g \circ S -E^\mathcal{G} (g \circ S)\|^2 
&=\sum_{\Delta \in \mathfrak{D}_{\ell} } E^{\mathcal F} \|g \circ S-E(g \circ S|S \in \Delta)\|^2 \chi\{S \in \Delta\}\\
&=\sum_{\Delta \in \mathfrak{D}_{\ell} } E^{\mathcal F} \|g \circ S-E(g \circ S|S \in \Delta)\|^2 \chi^2\{S \in \Delta\}\\
& = \sum_{\Delta \in \mathfrak{D}_{\ell} } E^{\mathcal F} \|(g \circ S-E(g \circ S|S \in \Delta))\chi\{S \in \Delta\}\|^2 \chi\{S \in \Delta\}\\
&\leq d L^2 {\ell}^{-2}.
\end{align*}
In particular,
\begin{equation} \label{eq:long_est_proof_term1}
E^{\mathcal F}\|g \circ S -E^\mathcal{G} \phi(X) \|^2
\leq d L^2 {\ell}^{-2}.
\end{equation}

On the other hand, in much the same way as in the proof of Proposition \ref{Prop:Est_error_cond},  we have
\begin{equation*}
E^\cF( \| E_n^S \phi(X) - E^{\mathcal G} \phi(X)\|^2) = \sum_{\Delta \in \mathfrak{D}_{\ell}} P(S\in \Delta) \sup_{\|h\|\leq 1} |E_n(h(X) | S \in \Delta)  - E(h(X) |  S \in \Delta)|^2.  
\end{equation*}
Let $U := (X,S)$ and define the push forward measure $\nu:=P \circ U^{-1}$ of $P$ onto $\mathbb X \times \mathbb S$ under $U$. Set $\nu_n:=\frac{1}{n}\sum_{i=1}^n\delta_{(X_i,S_i)}$ where $\delta_{(X,S)}$ denotes the measure that has point mass at $(X,S)$.
Define the projection map $\pi:\mathbb{X} \times \mathbb{S} \rightarrow \mathbb{X}  $ 
which maps a tuple $(x,s)\in \mathbb{X} \times \mathbb{S}$ to its first 
element so that $\pi((x,s))=x$. 
For each $h \in \cH$ such that $\|h\| \leq 1$  and every $\Delta \in \mathfrak{D}_{\ell}$ we obtain
\begin{align*}
|E_n(h(X) | S \in \Delta)  &- E(h(X) |  S \in \Delta)|^2 \\
&= |E_n(h(\pi(U)) | U \in \mathbb X \times \Delta)  - E(h(\pi(U)) |U \in \mathbb X \times \Delta)|^2\\
&=\Bigl|\int_{\mathbb X \times \Delta} h \circ \pi \, d \nu_n- \int_{\mathbb X \times \Delta} h \circ \pi \, d \nu\,\Bigr|^2.
\end{align*}
For each $\ell \in \mathbb N$ define $\mathfrak{C}_{\ell}:=\{\mathbb X \times \Delta: \Delta \in \mathfrak{D}_{\ell}\}$.  
By assumption, $\cH_{\mathfrak{C}}=\{h \times \chi {D}: h \in \cH,~\|h\|\leq 1,~D \in \bigcup_{\ell \in \mathbb N} \mathfrak{C}_{\ell}\}$ is $P$-Donsker and for $D \in \mathfrak{C}_\ell$, with $D = \mathbb{X} \times \Delta_i$ for $i \leq \ell$, $\nu(D) = P S^{-1}[\Delta_i] \geq b \ell^{-d}$. For a given $\alpha \in (0,1/2)$ let 
$\ell$ be $\lfloor n^{\alpha/d} \rfloor$  so that $\ell^{-d} \geq n^{-\alpha}$. Similarly to \cite[Proposition~3.2]{SG} it follows that there exists a constant $M$ such that 
for all $n\geq 1$ and corresponding $\ell$,
\begin{align*}
\sup_{\|h\|  \leq 1}\sup_{C \in \mathfrak{C}_{\ell}}
\bigl|\int_{C} h \circ \pi  \, d \nu_n- \int_{C} h \circ \pi \, d \nu \, \bigr| \leq 
2M \ell^dn^{-1/2}/b.
\end{align*}
Thus,
\begin{align} \label{eq:long_est_proof_term2} 
E^\cF( \| E_n^S \phi(X) - E^{\mathcal G} \phi(X)\|^2) \leq 4M ^2\ell^{2d}/n b^2.
\end{align}
Using Equation \eqref{eq:long_est_proof_term1} and 
\eqref{eq:long_est_proof_term2} as well as 
the Cauchy-Schwarz inequality for conditional expectations we obtain,
\begin{align*}
E^\cF( \| E_n^S \phi(X) - E^S \phi(X)\|^2) 
&\leq E^\cF( \| E_n^S \phi(X) - E^{\mathcal G} \phi(X)\| + \| E^{\mathcal G} \phi(X) - E^S \phi(X)\|)^2\\
&\leq 4M ^2\ell^{2d}/nb^2+4M\sqrt{d}L\ell^{d-1}n^{-1/2}/b +dL^2{\ell^{-2}}.
\end{align*}
Because $\ell = \lfloor n^{\alpha/d} \rfloor $ the upper-bound becomes
\[
4M^2n^{2\alpha-1} /b^2+dL^2/(n^{-\alpha/d}-1)^2+4M\sqrt{d}Ln^{\alpha(1-\frac{1}{d})-\frac{1}{2}} /b.
\]
We claim that the rate of convergence in $n$ is optimized by  $\alpha^*=d / 2(d+1)$: For $\alpha \geq \alpha^*$ we have 
\begin{align*}
&2\alpha-1 \geq \alpha(1-1/d)-1/2\geq -2\alpha/d
\end{align*}
and the dominant term $2\alpha -1$ is minimized at $\alpha^*$. On the other hand, for $\alpha \leq \alpha^*$,  
\begin{align*}
&2\alpha-1 \leq -2\alpha/d\text{\enspace ~and~ \enspace}\alpha(1-1/d)-1/2\leq -2\alpha/d.
\end{align*}
In this case the dominant term is also minimized for $\alpha^*$.
Therefore, we must set 
\[ \ell^*=\lfloor n^{\alpha^*/d} \rfloor= \lfloor n^{\frac{1}{2(d+1)}} \rfloor. \]
\end{proof}

\section{Solution to the oblivious kernel ridge regression optimization problem} \label{app:ridge_reg_opt}
Define $\mb z_i:=\left (\langle \mb Z_1, \mb Z_i \rangle \enspace \cdots \enspace \langle  \mb Z_n , \mb Z_i\rangle \right)^\top,~i \in 1..n$, and observe that
\[
   \cO=\begin{pmatrix}
    \vert & \vert &\dots & \vert \\
    \mb z_1   & \mb z_2  &\dots&\mb z_n \\
    \vert & \vert&\dots&\mb \vert
\end{pmatrix}.
\]
Let $\hat f$ be the minimizer of the regularized least-squares error as given by  \eqref{eq:lse}. By  the representer theorem there exist scalars $\alpha_1,\ldots, \alpha_{n}$ such that $\hat f = \sum_{j=1}^{n} \alpha_j \mb Z_j$. 
It follows that $\langle \hat{f}, \mb Z_i \rangle  
=\sum_{j=1}^n \alpha_j \langle  \mb Z_j , \mb Z_i \rangle$ so that,
\begin{align}
 \sum_{i=1}^n (\langle \hat{f}, \mb Z_i \rangle - Y_i)^2 + \lambda \|\hat{f}\|^2= (  \cO \boldsymbol{\alpha}-\mb{y} ) ( \cO \boldsymbol{\alpha}-\mb{y} )^\top+\lambda \boldsymbol{\alpha}^\top \cO \boldsymbol{\alpha}\label{eq:objhatf}
\end{align}
where $\boldsymbol{\alpha}:=(\alpha_1,\dots,\alpha_n)^\top$ and  $\mb{y}:=(Y_1,\dots,Y_n)^\top$. Noting that $\hat{f}$ is the minimizer, and thus taking the gradient of \eqref{eq:objhatf} with respect to  $\boldsymbol{\alpha}$ we obtain,
\[\nabla_{\boldsymbol{\alpha}} \Bigl(   (  \cO \boldsymbol{\alpha}-\mb{y} ) ( \cO \boldsymbol{\alpha}-\mb{y} )^\top+\lambda \boldsymbol{\alpha}^\top \cO \boldsymbol{\alpha} \Bigr)=0.\]

Solving for $\boldsymbol{\alpha}$ and noting that $\cO$ is symmetric, we obtain
\begin{align*}
\boldsymbol \alpha 
&= \cO^{-1}\Bigl (\cO^{\top}+\lambda I\Bigr)^{-1}\cO^{\top}\mb{y}\\
&=\cO^{-1}\Bigl (\cO^{\top}+\lambda I\Bigr)^{-1}\cO\mb{y} &\text{since $\cO$ is symmetric}\\
&=\cO^{-1}\Bigl (\cO^{\top}+\lambda I\Bigr)^{-1}(\cO^{-1})^{-1}\mb{y} \\
&=\Bigl( \cO^{-1}\Bigl (\cO^{\top}+\lambda I\Bigr)\cO\Bigr)^{-1}\mb{y}\\
&=\Bigl( (\cO^{-1}\cO +\lambda \cO^{-1})\cO\Bigr)^{-1}\mb{y} &\text{since $\cO$ is symmetric}\\
&= (\cO+\lambda I)^{-1} \mb{y}.
\end{align*}

\begin{algorithm}[t]
\caption{Generating the oblivious kernel matrix; the sum over an empty index set is treated as $0$}
\label{alg:obl}
\begin{algorithmic}
   \STATE {\bfseries Input:} data $(x_1,s_1),\ldots (x_{2n}, s_{2n})$, disjoint sets $A_1,\ldots, A_{\ell}$ which cover $\mathbb{S}$
	\STATE set $M = \sum_{i=n+1}^{2n} \sum_{j={n+1}}^{2n} k(x_i,x_j)/n^2$
   	\STATE set $\mathcal{\altI}_i = \emptyset,~i \in 1,\dots, \ell $
	\FOR{$i=n+1$ {\bfseries to} $2n$}   
	\STATE find index $u$ such that $s_i \in A_u$
	\STATE update $\mathcal{\altI}_u \leftarrow \mathcal{\altI}_u \cup \{i\}$
	\ENDFOR
	\FOR{$i=1$ {\bfseries to} $n$}
	\STATE set $\rho_i = \sum_{u=n+1}^{2n} k(x_i,x_u) /n$
	\FOR{$a=1$ {\bfseries to} $l$}	
	\STATE set $\xi_{i,a} = \sum_{u\in \mathcal{\altI}_a} k(x_i,x_u)/|\mathcal{\altI}_a|$
	\ENDFOR
	\ENDFOR 	
	\FOR{$a=1$ {\bfseries to} $l$}
	\STATE set $\tau_a = \sum_{u \in \mathcal{\altI}_a} \sum_{v=n+1}^{2n} k(x_u,x_v)/(n|\mathcal{\altI}_a|)$
	\FOR{$b=1$ {\bfseries to} $l$} 
\STATE set	$o_{a,b} = \sum_{u \in \mathcal{\altI}_a, v\in \mathcal{\altI}_b} k(x_u,x_v)/(|\mathcal{\altI}_a| |\mathcal{\altI}_b|)$
	\ENDFOR
	\ENDFOR
	\FOR{$i=1$ {\bfseries to} $n$}
		\FOR{$j=i$ {\bfseries to} $n$}
		\STATE set $a$ such that $s_j \in A_a$	
		\STATE set $b$ such that $s_i \in A_b$	
		\STATE set $\mathcal{O}_{i,j} = k(x_i,x_j)	 - \xi_{i,a} - \xi_{j,b}	+ o_{a,b}  + M - \rho_i -\rho_j -\tau_a -\tau_b$	
		\STATE set $\mathcal{O}_{j,i} = \mathcal{O}_{i,j}$
		\ENDFOR	
	\ENDFOR
\STATE {\bfseries Return:} $\mathcal{O}$
\end{algorithmic}
\end{algorithm}

\section{Algorithms} \label{suppl:alg}
We discuss three algorithms in this section: an algorithm to calculate the oblivious kernel matrix (Section \ref{suppl:alg_obl_matrix}), an algorithm to calculate $\langle \mZ, \mZ_i \rangle$ which is needed for prediction (Section \ref{suppl:alg_obl_pred}), and an algorithm to calculate $W$, the projection of $\mZ_i$ onto $\mathcal{M}$, which also allows us to estimate the distance between $\mZ$ and $\mathcal{M}$
(Section \ref{suppl:alg_W}).

\subsection{Calculating the oblivious kernel matrix} \label{suppl:alg_obl_matrix}
We start by deriving the algorithm for calculating the oblivious matrix. The result algorithm is summarized in Algorithm \ref{alg:obl} on page \pageref{alg:obl}. Throughout we assume that $A_1,\ldots, A_l$ is a partition of $\mathbb{S}$ and we assume that $2n$ samples $(X_i,S_i)$ are available.  The algorithm splits the data into two parts of size $n$ and uses the 
samples $n+1, \ldots, 2n$ to estimate the conditional expectation.  The remaining $n$ samples are then used to generate the features $\mZ_i$, $i=1,\ldots,n$. The features $\mZ_i$ will not be explicitly stored. The only thing that will be stored is the oblivious matrix $\mathcal{O}$. To calculate the oblivious matrix we only need kernel evaluations. To see this consider any  $i\leq n$, then
\begin{align*}
\mZ_i &= \phi(X_i) - E_n^{S_i} \phi(X) = \phi(X_i) - \sum_{u=1}^l E_n(\phi(X) |S\in A_u) \times \chi\{S_i \in A_u\}.
\end{align*}
For $u=1,\ldots,l$ let 
\[
 N_u = \sum_{v=n+1}^{2n} \chi\{ S_v \in A_u\}
\]
be the number of samples with indices within $n+1,\ldots, 2n$ that fall into set $A_u$. The estimate of the elementary conditional expectation is 
\[
E_n(\phi(X) |S \in A_u) = \frac{1}{N_u} \sum_{v=n+1}^{2n} \phi(X_v) \times \chi\{S_v \in A_u\},
\]
which attains values in $\cH$.

Now consider the inner product between $\mZ_i$ and $\mZ_j$, $i,j \leq n$:
\begin{align*}
\langle \mZ_i,\mZ_j \rangle =& \langle \phi(X_i),\phi(X_j) \rangle - \langle \phi(X_i), E_n^{S_j} \phi(X) \rangle - \langle E_n^{S_i} \phi(X), \phi(X_j) \rangle \\ 
&+\langle E_n^{S_i} \phi(X),E_n^{S_j} \phi(X) \rangle + \langle \phi(X_i), E_n(\phi(X)) \rangle\\ 
& + \langle  E_n(\phi(X)), \phi(X_j) \rangle - \langle E_n^{S_i} \phi(X), E_n(\phi(X)) \rangle  \\
&- \langle  E_n(\phi(X)), E_n^{S_j} \phi(X) \rangle    + \langle E_n(\phi(X)), E_n(\phi(X)) \rangle.
\end{align*}
This reduces to calculations involving only the kernel function and no other functions from $\cH$. In detail,
\[
\langle \phi(X_i),\phi(X_j) \rangle = k(X_i,X_j),
\]
and
\[
\langle \phi(X_i), E_n^{S_j} \phi(X) \rangle = 
\sum_{u=1}^l \langle \phi(X_i), E_n(\phi(X) |S\in A_u) \rangle \times \chi\{S_j \in A_u\},
\]
where 
\begin{align*}
\langle \phi(X_i), E_n(\phi(X) |S \in A_u) \rangle 
&=  \frac{1}{N_u} \sum_{l=n+1}^{2n} \langle \phi(X_i), \phi(X_l)\rangle  \times \chi\{S_l  \in A_u\} \\
&= \frac{1}{N_u} \sum_{l=n+1}^{2n}  k(X_i, X_l)  \times \chi\{S_l \in A_u\}.
\end{align*}
The inner product $\langle  E_n(\phi(X) |S \in A_u), \phi(X_j) \rangle$ can be calculated in the same way. Furthermore,
\begin{align*}
&\langle E_n^{S_i} \phi(X),E_n^{S_j} \phi(X) \rangle = \sum_{u=1}^l \sum_{v=1}^l  
\langle  E_n(\phi(X) |S \in A_u),   E_n(\phi(X) |S \in A_v) \rangle
\times \chi\{ S_i \in A_u, S_j \in A_v\}
\end{align*}
and
\begin{align*}
&\langle  E_n(\phi(X) |S\in A_u),   E_n(\phi(X) |S\in A_v) \rangle \\
&= \frac{1}{N_u N_v} \sum_{l=n+1}^{2n} \sum_{m=n+1}^{2n} \langle \phi(X_l), \phi(X_m) \rangle \times \chi\{S_l \in A_u, S_m \in A_v\} \\
&= \frac{1}{N_u N_v} \sum_{l=n+1}^{2n} \sum_{m=n+1}^{2n} k(X_l,X_m)  \times \chi\{S_l \in A_u, S_m \in A_v\}.
\end{align*}
The terms involving $E_n(\phi(X)) = (1/n) \sum_{i=n+1}^{2n} \phi(X_i)$ are reduced in a similar way to kernel evaluations. Combining these calculations leads to  Algorithm \ref{alg:obl}.

\subsection{Prediction based on oblivious features} \label{suppl:alg_obl_pred}
 To be able to predict labels for new observations $(X,S)$ in a regression or classification setting we need to transform $(X,S)$ into an oblivious feature $\mZ$. The approach to do is the same as for the training data. In particular, the conditional expectation estimates $E_n^{S_j} \phi(X)$ are needed to transform $(X,S)$ into $\mZ$.
For kernel methods $\mZ$ itself is never calculated explicitly but it appears in algorithms in the form of inner products 
$\langle \mZ, \mZ_i \rangle$,
where $i\leq n$ and $\mZ_i$ are the oblivious features corresponding to the training set. These inner product can be calculated in exactly the same way as the inner products $\langle \mZ_i, \mZ_j \rangle$ in  Section \ref{suppl:alg_obl_matrix}.

\subsection{Projecting the oblivious features onto the manifold} \label{suppl:alg_W}
The quadratic distance between $\mZ$, or more precisely $\mZ(\omega)$, and $\mathcal{M}$ in $\cH$ is equal to
\[
\inf_{x \in \mathbb{X}} \|\mZ  - \phi(x)\|^2 = \|\mZ\|^2 + 
\inf_{x\in \mathbb{X}} ( k(x,x) -2  \langle \mZ, \phi(x) \rangle).
\]
The constant  $\|\mZ\|^2$ is of no relevance and we are looking for a minimum (when this is well-defined) of the function 
\[
f(x) =  k(x,x) -2  \langle \mZ, \phi(x) \rangle
\]
in $\mathbb{X}$. Using the conditional expectation $E_n^{S} \phi(X)$ and $\mZ = \phi(X) - E_n^S \phi(X) + E(\phi(X))$ we can rewrite $f(x)$ as 
\[
f(x) = k(x,x) - 2 (k(X,x)  - E_n^S k(X , x)  +  E(k(X,x) ).
\]
The function $f$ is $(\alpha,L)$-H\"older continuous whenever $k(x,\cdot)$ is $(\alpha,L')$-H\"older-continuous for all $x\in \mathbb{X}$ with $L = 8L'$, since then
\begin{align*}
|f(x) - f(y)| &\leq |k(x,x) - k(y,y)| +2   (|k(X,x) - k(X,y)|  + E_n^S |k(X,y) - k(X,x)|  \\ 
&\hspace{8cm} +  E |k(X,x) -k(X,y)| )  \\
&\leq | k(x,x) - k(x,y)| + |k(x,y) - k(y,y)|  +  6L' \|y-x\|^{\alpha} = 8L' \|y-x\|^{\alpha}.
\end{align*}
This property of $f$ is useful because various kernel functions are H\"older-continuous and efficient algorithms are available to optimize H\"older-continuous functions. In particular, there exist classical global optimization algorithms \citep{VAN97} and bandit algorithms \citep{MUN14} for this task.

The projection of $\mZ$ onto $\mathcal{M}$ can also be used directly to approximate $d_n(\mZ+h^*,\cM)$ and, by applying Proposition \ref{prop:dist_chaining}, to estimate $d(\mZ+h^*,\cM)$.

\end{document}